\documentclass[12pt]{class2020} % Include author names

\usepackage{tcolorbox}
\usepackage{cases}
\usepackage{mathtools}
\usepackage{bbold}
\usepackage{caption}
\usepackage{graphicx,subcaption}

\DeclarePairedDelimiter\floor{\lfloor}{\rfloor}
\newtheorem{claim}[theorem]{Claim}

\title[New Potential-Based Bounds for Prediction with Expert Advice]{New Potential-Based Bounds for Prediction with Expert Advice}
\usepackage{times}
\coltauthor{%
 \Name{Vladimir A.\ Kobzar} \Email{vladimir.kobzar@nyu.edu}\\
 \addr Center for Data Science, New York University, 60 Fifth Ave., New York, New York
 \AND
 \Name{Robert V.\ Kohn} \Email{kohn@cims.nyu.edu}\and
  \Name{Zhilei Wang} \Email{zhilei@cims.nyu.edu}\\
  \addr Courant Institute of Mathematical Sciences, New York University, 251 Mercer St., New York, New York %
}

\begin{document}

\maketitle

\begin{abstract}%
This work addresses the classic machine learning problem of online prediction with expert advice. We consider the finite-horizon version of this zero-sum, two-person game. 
Using verification  arguments  from optimal control theory, we view the task of finding  better lower and upper bounds on the value of the game (regret) as the problem of finding better sub- and supersolutions of certain partial differential equations (PDEs).
These sub- and supersolutions serve as the potentials for player and adversary strategies, which lead to the corresponding bounds. To get explicit bounds, we use closed-form solutions of specific PDEs. Our bounds hold for any given number of experts and horizon; in certain regimes (which we identify) they improve upon the previous state of the art. For two and three experts, our bounds provide the optimal leading order term.
\end{abstract}

%\begin{keywords}%
% Online learning,  expert advice framework, regret minimization, fixed horizon,  upper and lower bounds, potential-based strategies, subsolutions and supersolutions of partial differential equations, optimal control, dynamic programming, verification argument, linear heat equation, comb adversary%
%\end{keywords}

\section{Introduction}
\label{sec:intro} 
The classic machine learning problem of online prediction with expert advice (the \textit{expert problem}) is a repeated two-person zero-sum game with the following structure. At each round,  the predictor (\textit{player}) uses guidance from a collection of \textit {experts} with the goal of minimizing the difference (\textit {regret}) between the player's loss and that of the best performing expert in hindsight. The environment (\textit{adversary}) determines the losses of each expert for that round. The player's selection of the experts and the adversary's choice of the loss for each expert are revealed to both parties, and this prediction process is repeated until the final round. 

We will focus on the following representative definition of this problem, which mirrors (up to translation and rescaling of the loss) the version considered in recent work on optimal strategies~\citep{gravin16, abbasi}.   
 \begin{tcolorbox} 
  \textit{Prediction with expert advice:} 
  At each period $t \in [T]$, (a) the \textit {player} determines which of the $N$ \textit{experts} to follow by selecting a discrete probability distribution $p_t \in \Delta_N$; (b) the \textit{adversary} allocates losses to the experts by selecting a probability distribution $a_t$ over the hypercube $ [-1,1]^N$; and (c) the expert losses $q_t \in [-1,1]^N $ and the player's choice of the expert $I_t \in [N]$ are sampled from $a_t$ and $p_t$, respectively, and revealed to both parties.
 \end{tcolorbox}
In this setting $a =(a_t)_{t\in [T]} $ and $p=(p_t)_{t\in [T]}$ refer to, respectively, the  \textit{adversary} and \textit{player strategies} or simply the adversary and player. The player strategy may be known to the adversary, and vice versa. In general, each strategy at time $t$ can depend on the history of losses and choices of the expert in previous periods. However, the flow of information above implies that, conditioned on the history,  $q_t$ and $I_t$ are independent.   We consider the \textit {finite horizon} version of the problem, where the number of periods $T$ is fixed and the regret is  $R_T(p,a) =  \mathbb E_{ p,a} \left [ \sum_{t \in [T]}  q_{I_t,t}  - \min_{ i } \sum_{t\in [T]} q_{i,t} \right]$.
 
Numerous strategies attain vanishing per round regret.  For example, the \textit {exponentially weighted forecaster} $p^{e}$ provides the upper bound $ \max_a R_T(p^{e},a) \leq  \sqrt {2 T \log N}$.  Also for all $\epsilon >0$, there exist $N$ and $T$ sufficiently large, such that the \textit {randomized adversary} $a^r$ (which assigns $1$ or $-1$ to each component of $q$ independently with equal probability) approaches that bound: $(1-\epsilon)  \sqrt {2 T \log N} \leq   \min_p R_T (p,a^r)$.\footnote{See \citet{cesa-bianchi97} and Theorems 2.2 and 3.7 in \citet{cesa-bianchi06}. These results are rescaled here to apply to $[-1,1]^N$, instead of $[0,1]^N$, losses.}  
  
A minmax optimal player (\textit{optimal player})  is a player that minimizes the regret over all possible adversaries and a minmax optimal adversary  (\textit{optimal adversary}) is an adversary  that maximizes the regret over all possible players. Thus, $p^e$ and $a^r$ are optimal asymptotically in $T$ and $N$. 

Nonasymptotic optimal strategies have been determined explicitly using random walk methods for $N=2$, and, up to the leading order term, for $N=3$~\citep{cover, gravin16, abbasi}. For general $N$, optimal strategies can be found using dynamic programming and depend only on the cumulative losses of each expert and the remaining time, rather than the full history of adversary's and/or player's choices \citep {cesa-bianchi97}. \citet{luo} determined optimal strategies in the version of the problem where the adversary's choice of losses is restricted to the set of standard basis vectors. However, optimal strategies for the original game have not been determined explicitly.

In a related line of work, strategies that are optimal asymptotically in $T$ have been determined by PDE-based methods.  For $N=2$, \citet{Zhu} established that the value function is given by the solution of a 1D linear heat equation, which provides a continuous perspective on the random walk characterization of the non-asymptotic problem.  \citet{drenska2019prediction} showed that for any $N$, the value function, in a scaling limit, is the unique solution of an associated nonlinear PDE.  \citet{bayraktar2019b} found closed-form solutions of the PDEs for $N = 3$ and $4$.

Due to the complexity of determining optimal strategies for an arbitrary $N$, it is common to use potential functions to bound the regret above. For example, $p^e$ uses the logarithm of the sum of the exponentials of regret with respect to each expert as the potential; the corresponding upper bound is obtained by bounding the evolution of this potential for all possible adversaries. \citet{chaudhuri} and \citet{luo_ada} proposed other potential-based player algorithms for variations of the expert problem with different notions of regret and/or additional structure.

\citet{rakhlin2012} proposed a principled way of deriving potential-based player strategies by bounding above the value function in a manner that is consistent with its recursive minmax form. \citet{rokhlin} suggested using supersolutions of the asymptotic PDE as potentials for player strategies leading to upper bounds. The present paper extends these ideas by applying related arguments to broad classes of potentials, and by providing lower as well as upper bounds.

Adversary strategies have been commonly studied as random processes. For example, for any  $N$, $a^r$ guarantees that the leading order regret is bounded below by the expectation of the maximum of $N$ i.i.d. Gaussians.\footnote{See Theorem 3.7 in \citet{cesa-bianchi06}.}
This guarantee is based on the central limit theorem and is therefore asymptotic in $T$. Nonasymptotic lower bounds have been established using random walk methods. \citep{orabona, gyorgy}. 

The player's and the adversary's selection of strategies is fundamentally a problem of optimal control. Adopting such a viewpoint, in this paper we propose a control-based framework for designing strategies for the expert problem using sub- and supersolutions of certain PDEs. Our principal conceptual advances are the following.  
\begin {enumerate} 
\item The potential-based framework is extended to adversary strategies, leading to lower bounds (Section \ref{sec:lb_fixed}).
\item The task of finding better regret bounds reduces to the mathematical problem of finding better sub- and supersolutions of certain PDEs (See Equations \eqref{eq:pde_fixed_lb} and \eqref{eq:pde_fixed_ub}). 
\item Our bounds hold for any given number of experts and are nonasymptotic in $T$; their rate of convergence to the asymptotic (in $T$) value is determined explicitly using error estimates similar to those used in finite difference schemes in numerical analysis.  (Theorems \ref{thm:fixed_lb} and \ref{thm:fixed_ub}).
\end{enumerate}

These conceptual advances not only provide a fresh perspective on the expert problem, but also lead in some cases to improved bounds.  Specifically, we apply our framework to two classes of potentials. The first class is discussed in Section \ref{sec:heat_fixed}, where we use  classical solutions of the linear heat equation with suitable diffusion factors as lower and upper bound potentials.  The leading order term of the resulting lower bound is the expectation of the maximum of $N$ i.i.d. Gaussians with mean zero, and is therefore similar to the existing lower bound given by $a^r$. However, the constant factor of the leading order term (i.e., the variance of the Gaussians) is state-of-the-art. Additionally, we improve the bounds on the higher order (error) terms (Section \ref{sec:related_work}).

A second class of potentials is discussed in Section \ref{sec:max_fixed}. They are closed-form solutions of a nonlinear PDE where the spatial operator involves the largest diagonal entry of the Hessian.  For up to three experts, the lower and upper bounds obtained using this potential  match to leading order as the number of time steps approaches infinity. Therefore, the corresponding strategies are optimal to leading order.  The same leading order result for three experts was determined in \citet{abbasi};  our approach, however, provides a smaller error term. Also, for small $N$ and relatively large number of time steps our upper bound is tighter than the one obtained using $p^e$ (Section \ref{sec:related_work}).

 \section{Notation} 
 
 The ``spatial variables'' and ``spatial derivatives'' of a function $u(x,t)$ are $x \in \mathbb R^N $
and the derivatives of $u$ with respect to $x$. For a multi-index $I$, $\partial_{I}$ refers to the partial derivative and $dx_I$ refers to the differential with respect to the spatial variable(s) in $I$, and $d \hat x_I$ refers to the differential with respect to all except the spatial variables in $I$.  $D^2 u$, $D^3 u$ and $D^4 u$  refer to the Hessian, 3rd derivative, and 4th derivative of $u$ with respect to $x$ (which are 2nd order, 3rd order and 4th order tensors respectively); the associated multilinear forms $\langle D^2 u \cdot q, q \rangle $, $D^3 u [q,q,q]$, $D^4 u [q,q,q,q]$ are $\sum_{i,j} \partial_{ij} u~ q_{i} q_{j} $, $\sum_{i,j,k} \partial_{ijk} u~ q_{i} q_{j} q_k $ and $\sum_{i,j,k,l} \partial_{ijkl} u~ q_{i} q_{j} q_k q_l $.

Prediction with expert advice is a repeated two-person game. It is convenient to denote the time $t$ by nonpositive numbers such that the starting time is $T \leq -1$ and the final time is zero.  The vector $r_{\tau} = q_{I_{\tau}, \tau} \mathbb 1  -  q_{\tau}$ denotes the player's losses realized in round $\tau$ relative to those of each expert (\textit {instantaneous regret}) and the vector $x_t = \sum_{\tau <t } r_{\tau}$ denotes the player's cumulative losses realized before the outcome of round $t$ relative to those of each expert (cumulative regret or simply the \textit {regret}).

If $u$ is a function of space and time, subscripts $x$ or $t$ denote partial derivatives (so $u_x$ and $u_t$ are first derivatives and $u_{xx}$, $u_{xt}$ and $u_{tt}$ are second derivatives). In other settings, the subscript $t$ is an index; in particular, our adversary and player strategies at time $t$ are $a_t$ and $p_t$ and the expert losses and player's choice at time $t$ are $q_t$ and $I_t$. When no confusion will result, we sometimes omit the index $t$, writing for example $q$ rather than $q_t$; in such a setting, $q_i$ refers to the $i$th component of $q_t$. 

If $u$ is a function, $\Delta u = \sum_i \frac{\partial^2 u}{\partial x_i^2}$ is its Laplacian; however, the standalone symbol $\Delta_N$ refers to the set of probability distributions on $\{1, ..., N\}$.
$[T]$ denotes the set  $\{1, ..., T\}$ if $T \geq 1$ or $\{T, ..., -1\}$ if $T \leq -1$.  $\mathbb 1$ is a vector in $\mathbb R^N$ with all components equal to 1,  but $\mathbb 1_S$ refers to the indicator function of the set $S$.  

A \textit{classical solution} of a partial differential equation (PDE) on a specified region is a solution such that all derivatives appearing in the statement of the PDE exist and are continuous on the specified region.
 
\section{Lower Bounds}
\label{sec:lb_fixed}  Our lower bounds are associated with well-chosen strategies for the adversary. We shall consider adversary strategies that are \textit{Markovian}, in the sense that the strategy at time $t$ can depend only on the cumulative regret $x$ and time $t$.  For a given adversary $a$, it is natural to consider the associated value function $v_a(x,t)$, defined as the final-time regret achieved by the adversary (assuming the player behaves optimally) if the prediction game starts at time $t$ with cumulative regret vector  $x$. It is characterized by a dynamic program (DP):\footnote {Our use of dynamic programming is related to the arguments used in  Section 3 in \citet {cesa-bianchi97} to show that the optimal strategies are Markovian. Our use is different, however, (and simpler) since we assume from the start that the adversary's strategy is Markovian.}
\begin {align}
  v_a(x, 0) &=  \max_{ i } x_i ~~\text{and}~~ v_a(x, t) = \min_{p}  \mathbb E_{a_t, p} ~v_a(x+r, t+1)  ~\text{for} ~t \leq -1
  \label{eq:va_dp}
\end{align}
Working backward in time, the DP determines the player's optimal strategy at each time. It is clearly Markovian, in the sense that this strategy depends only on the time $t$ and the cumulative regret $x$ at that time.

In the context of lower bounds, we shall consider only adversaries that assign the same expectation of each component of $q$:  $\mathbb E_{a_t} q = c_t \mathbb 1$ for some $c_t \in [-1,1]$ and all $t < -1$ (\textit{balanced adversaries}).  To bound $v_a$ below, we introduce the following class of potential functions, or simply \textit{potentials}.  As described more fully in Section \ref{sec:relationship_to_dk}, such a potential bounds below the minimax optimal (asymptotically in T) value  because the potential is a \textit{subsolution} of the nonlinear PDE \eqref{eq:drenska_pde} obtained in \citet{drenska2019prediction}.
\begin{tcolorbox} 
\textit{A lower bound potential} is a function $u: \mathbb R^N \times \mathbb R_{\leq 0}  \rightarrow \mathbb  R$ such that for each $x \in \mathbb R^N$ and $t <0$, there is a balanced strategy $a_t$ on $ [-1, 1]^N$ ensuring that $u$ is a classical solution of 
\begin{subnumcases}{\label{eq:pde_fixed_lb}}
 u_t +  \frac {1}{2}  \mathbb E_{a_t} \langle D^2 u \cdot q, q \rangle \geq 0 \label{eq:pde_ineq_lb}\\
 u(x,0) \leq \max_i x_i ~~\text{and}  ~~u(x+ c \mathbb 1, t)=u(x, t)+c  \label{eq:pde_boundary_lb}
\end{subnumcases}
At $t < -1$, an \textit{adversary associated with $u$} is a balanced strategy $a_t$ such that \eqref{eq:pde_ineq_lb} is satisfied at $(x, t+1)$. At $t = -1$, any distribution $a_{-1}$ over $[-1,1]^N$ may be used. 
 \end{tcolorbox}
We prove in Appendix \ref{app:thm_fixed_lb}, using induction backward in time, that this potential bound below the adversary's optimal value $v_a$, modulo an ``error" term $E(t)$ which can be estimated explicitly.  This provides a lower bound on regret since $v_a(0, T) = \min_p R_T(a,p)$. Note that while the definition of $v_a$ involves an optimization over the player strategy $p$, the definition of $u$ does not. Examination of the proof (in Appendix \ref{app:thm_fixed_lb}) reveals that our lower bound is insensitive to $p$  because the adversary strategy $a$ is balanced. 
\begin{theorem} [Lower bound]
 \label{thm:fixed_lb}
 Let $u$ be a lower bound potential and let $v_a$ be the value function of the associated adversary $a$.  Then, $u(x,t)-E(t) \leq v_a(x,t)$ where the error term $E(t) = C+ \sum _{\tau = t}^{-2}  K(\tau)$ is computed using: (i) a bound on the decrease of $u$ at the last period, which is a constant $C$ satisfying  $u(x,-1) - \min_{p} \mathbb E_{a_{-1},p} ~u  (x + r,  0)  \leq C$ for all $x$, and (ii) an error estimate $K$ of the Taylor approximation of $u$ in the earlier periods.  If $u_{t}(x,\cdot)$ and $D^2u(\cdot, \tau+1) $ are Lipschitz continuous, then any function K satisfying $
 \frac {1}{2}\text{ess sup}_{  \bar \tau \in [\tau, \tau+1] }   u_{tt}(x,\bar \tau ) + \frac {1}{6}   \text{ess sup}_{y \in [x, x -q]} ~D^3u(y,\tau+1) [q,q,q]   \leq K(\tau)$ for all $\tau \in [t,  -2]$, all $q$ in the support of $a_{\tau}$ and all $x$, may be used to compute E(t).
 \end{theorem}
If the adversary assigns the same probability to $q$ and $-q$ to each $q$ in its support (a \textit{symmetric adversary}) and the potential is smooth enough, there is an alternative estimate for the error term, proved in Appendix \ref{app:fixed_lb_lipschitz}, which in some examples gives a better result. 
\begin{proposition} [Symmetric adversary and smooth potential]
 \label{rem:fixed_lb_lipschitz}  If the adversary $a$ associated with $u$ is symmetric, and $D^3 u(\cdot,t+1)$ exists and is Lipschitz continuous, then in Theorem \ref{thm:fixed_lb} any function K satisfying $\frac {1}{2}\text{ess sup}_{  \bar \tau \in [\tau, \tau+1] }    u_{tt}(x,\bar \tau)   -\frac {1}{24}  \text{ess inf}_{y \in [x, x -q]} ~D^{4}u(y,\tau+1) [q,q,q,q]   \leq K(\tau)$ for  all $\tau \in [t, -2]$,  all $q$ in the support of $a_{\tau}$ and all $x$, may be used to compute E(t).  
\end{proposition}
In what follows, we will apply our framework to obtain a fresh perspective on the best existing lower bounds and we will obtain improved lower bounds. Specifically, in Example \ref{ex:randomized_lb}, using the heat potential $\varphi$ given by \eqref{eq:heat_potential} with the diffusion factor $\kappa = \frac{1}{2}$, we recover the  well-known asymptotic lower bound associated with the randomized adversary $a^r$. We also show that the so-called comb adversary $a^c$ does at least as well as $a^r$ at leading order in the limit as $|T| \rightarrow \infty$. By applying Proposition \ref{rem:fixed_lb_lipschitz}, we obtain explicit nonasymptotic bounds for both adversaries.  In Example \ref{ex:heat_bounds}, we introduce a new heat adversary $a^h$, associated with the heat potential with a higher diffusion factor $\kappa_h > \frac{1}{2}$, which improves upon the lower bound associated with $a^r$ and $a^c$. For $N=2$,  $a^h$ is asymptotically optimal.\footnote{For $N=2$, $a^c$ is the same as $a^h$.} 

Section \ref{sec:max_fixed} applies our framework to an adversary associated with the new max potential $\psi$ given by \eqref{eq:max_pde_soln}. This adversary is asymptotically optimal for $N=2$ and $3$.\footnote{For $N=2$, the adversary associated the max potential is identical to $a^c$ and $a^h$.}

\section {Upper Bounds} 

Our upper bounds are associated with strategies for the player given by the gradient of specific potentials. We shall only consider potentials that can depend, at time $t$, only on the cumulative regret $x$ and time $t$. Consequently, our player strategies are Markovian. In parallel to the discussion above, for a given player $p$, we consider the value function $v_p(x,t)$ defined as the final-time regret achieved by this player (assuming the adversary behaves optimally) if the prediction game begins at time $t$ with cumulative regret vector $x$. It is characterized by the following DP:
\begin {align}
  v_p(x, 0) =  \max_{ i } x_i ~~\text {and}~~ v_p(x, t) = \max_{a}  \mathbb E_{a, p_t} ~v_p(x+r, t+1)  ~\text{for} ~t \leq -1\label{eq:vp_dp}
\end{align}
Working backward in time, this DP determines the adversary's optimal strategy at each time, and this strategy is also Markovian. 

To bound $v_p$ above, we introduce the following class of potentials.  As described more fully in Section \ref{sec:relationship_to_dk}, such a potential bounds above the minimax optimal (asymptotically in $T$) value  because the potential is a \textit{supersolution} of the PDE \eqref{eq:drenska_pde}.

\begin{tcolorbox} 
  \textit{An upper-bound potential} is a function $w: \mathbb R^N \times \mathbb R_{\leq 0}  \rightarrow \mathbb  R$, which is nondecreasing as a function of each $x_i$, and which is, for all $x \in \mathbb R^N$ and $t <0$ is a classical solution of 
\begin{subnumcases}{\label{eq:pde_fixed_ub}}
 w_t +  \frac {1}{2}\max_{q \in [- 1,1]^N} \langle D^2 w \cdot q, q \rangle \leq 0 \label{eq:pde_ineq_ub} \\
 w(x,0) \geq \max_i x_i   ~~\text{and} ~~w(x+ c \mathbb 1, t)=w(x, t)+c  \label{eq:pde_boundary_ub}
\end{subnumcases}
 \textit{The player strategy} $p$ associated with $w$ is: At  $t <-1 $, the player selects $p_t = \nabla w(x,t+1)$, and at $t = -1$, the player selects an arbitrary distribution  $p_{-1} \in \Delta_N$. 
 \end{tcolorbox}
\noindent At $t <-1 $, since $w$ is nondecreasing in each $x_i$, $p_{i,t} \geq 0$. Also $\sum_{i} \partial_i w =1$ by linearity of $w$ along $\mathbb 1$,  which implies that $\sum_{i} p_{i,t}=1$. Therefore, at $t <-1$, $p_t \in \Delta_N$  as well. 

The following Theorem is proved in Appendix \ref{app:thm_fixed_ub} using induction backward in time. It shows that an upper bound potential $w$ bounds above for the value function $v_p$, modulo an ``error" term $E(t)$. This provides an upper bound on the regret since $\max_a R_T(a,p) = v_p(0, T)$.  The argument (which is parallel to that for Theorem \ref{thm:fixed_lb}) uses Taylor expansion to estimate how $w$ changes as regret accumulates. The player strategy ensures that the first-order term of the Taylor expansion vanishes regardless of the adversary strategy $a_t$.  
\begin{theorem} [Upper bound] 
\label{thm:fixed_ub} Let $w$ be an upper bound potential and let $v_p$ be the value function of the associated player $p$.  Then,  $v_p(x,t) \leq w(x,t)+E(t)$ where the error term $E(t) = C+ \sum _{\tau = t}^{-2} K(\tau)$ is computed using: (i) the bound on the increase of $w$ at the last period, which is a constant $C$ satisfying 
 $\max_{a} \mathbb E_{a, p_{-1}} ~w  ( x + r,  0)  - w(x,-1)  \leq C$ for all $x $, and (ii) an error estimate K of the Taylor approximation of $w$ in the earlier periods. If $w_{t}(x,\cdot)$ and $D^2w(\cdot,\tau+1)  $  are Lipschitz continuous, then any function K satisfying $- \frac {1}{2} \text{ess inf}_{\bar \tau \in [\tau, \tau+1] }  w_{tt}(x,\bar \tau) -\frac {1}{6}   \text{ess inf}_{y \in [x, x -q]} ~  D^3w(y,\tau+1) [ q,q,q]  \leq K(\tau)$ for all $\tau \in [t,-2]$, all $q \in [- 1,1]^N$ and all $x$ may be used to compute $E(t)$.
\end{theorem}
If an upper bound potential has the form
\begin{align}
w(x,t) =  \Phi(x)+ct
\label{eq:stationary_players}
\end{align} for a constant $c$, the player $\nabla w(x)$ does not depend on time. Therefore, we can let the player strategy to be $\nabla w(x)$ at $t=-1$, instead of an arbitrary distribution. The following Proposition, proved in Appendix \ref{app:constant_w_t}, is similar to Theorem 1 in \citet {rokhlin}, and in this setting, the error term does not appear.
\begin{proposition} [Certain potentials]
 \label{rem:constant_w_t}  If, in the setting of Theorem \ref{thm:fixed_ub}, $w$ has the form \eqref{eq:stationary_players}, and the player strategy is $\nabla w(x)$ in all periods, then $v_{p}(x,t) \leq  w(x,t)$. 
\end{proposition}  
As an example, we recover the classic upper bound for the exponentially weighted forecaster $p^e$. Let the potential $w^e$ be given by $w^e (x,t) =\Phi(x) -  \frac {1}{2 } \eta t$ where  $\Phi(x) = \frac{1}{\eta} \log(\sum_{k \in[N]} e^{\eta x_k})$. In Appendix \ref{app:exp_ub}, we show that $\max_{q \in [- 1,1]^N} \langle D^2 \Phi \cdot q, q \rangle \leq \eta$.  Also $w^e(x,0) \geq \max_i x_i$, and $\Phi (x+c\mathbb 1) =  \Phi (x)+c$, which imply the same results  for $w^e$. Therefore, $w^e$ satisfies \eqref{eq:pde_fixed_ub} and Proposition \ref{rem:constant_w_t} provides the following result. 
\begin {example}[Exponential weights] 
For the value function $v_{p^e}$ of $p^e$, the following upper bound holds $v_{p^e}(x, t) \leq w^e(x, t)$. Taking $\eta = \sqrt {\frac {2  \log N}{|T|}}$ leads to the regret bound: $\max_a R(a,p^e)  \leq w^e(0, T) =\sqrt { 2|T| \log N}$.\footnote{This example provides the best known upper bound for $p^e$ and therefore gives a PDE perspective on Theorem 2.2 of \citet{cesa-bianchi06} (rescaled here to reflect $[-1,1]^N$ losses).} 
\end{example}

\section {Heat Potentials} 
\label{sec:heat_fixed}

In this section, we consider the \textit{heat potential} $\varphi$ given by
\begin{align}
\label{eq:heat_potential}
\varphi(x,t) &= \alpha \int  e^{-\frac  {\|y \|^2} {2\sigma^2} }\max_k (x_k-y_k) dy
\end{align}
where $\alpha = (2 \pi \sigma^2)^{-\frac{N}{2}}$ and $\sigma^2 = -2 \kappa t$. The linearity of the $\max$ function in the direction of $\mathbb 1$ implies that $\varphi(x+ c \mathbb 1, t)=\varphi(x, t)+c$.  This potential is the classical solution, on $\mathbb R^N \times \mathbb R_{<0}$, of the following linear heat equation
\begin{align*}
\begin{cases}
\varphi_t +\kappa \Delta \varphi =0  \\
\varphi(x,0) = \max_i x
\end{cases}
\end {align*}
 Therefore, $\varphi$ satisfies \eqref{eq:pde_boundary_lb} and \eqref{eq:pde_boundary_ub}. Let $G$ denote a  $N$-dimensional Gaussian vector with mean 0 and identity covariance. By the definition of the heat potential, 
\begin{align}
\varphi(0,T) = \sqrt {-2\kappa T} \mathbb E_G \max G_i
\label {eq:phi_at_final_time}
\end{align} 
Let $E_{l.b.}^\varphi$ denote the error term within the meaning of Theorem \ref{thm:fixed_lb} for the lower bound potential $\varphi$ with any $\kappa \in [\frac {1}{2},1]$ and any adversary supported on $\{\pm 1\}^N$. Appendix \ref{app:lb_err} shows that since $\varphi$ is smooth, by Proposition \ref{rem:fixed_lb_lipschitz} this term is $O(N\sqrt{N})$ uniformly in $t$. Theorem \ref{thm:fixed_lb} is also available and provides $E_{l.b.}^\varphi (t)=  O  (\sqrt {N \log N}+\sqrt{N} \log |t| )$. Therefore, $E_{l.b.}^\varphi(t)  =  O (  N\sqrt{N} \wedge \sqrt {N \log N}+\sqrt{N} \log |t|  )$. 

 Let $E_{u.b.}^\varphi$ denote the error term within the meaning of Theorem \ref{thm:fixed_ub} for upper bound potential $\varphi$ with $\kappa=1$.   Appendix \ref{app:ub_err} shows that $E_{u.b.}^\varphi(t)=O  (\sqrt {N \log N}+\sqrt{N} \log |t|)$.\footnote{While the asymptotic notation is used here for conciseness, the Appendices provide explicit error bounds.}

We consider the classic randomized adversary $a^r$ defined in Section \ref{sec:intro}.  Since it is symmetric, the mixed terms $\partial_{ij} \varphi q_i q_j$ have zero expectation, and consequently $ \mathbb E_{a^r}\langle D^2 \varphi \cdot q,q\rangle = \Delta \varphi$. Therefore, a lower bound potential $u^r = \varphi$ with $\kappa = \frac{1}{2}$ also satisfies \eqref{eq:pde_ineq_lb}, and we recover the classic asymptotic lower bound for $a^r$ with a new nonasymptotic error term in Example \ref{ex:randomized_lb}.  Moreover, since both inequalities in \eqref{eq:pde_fixed_lb} are satisfied with equalities, the proof of Theorem \ref{thm:fixed_lb} shows that the difference between $v_{a^r}$ and  $u^r$ is \textit{entirely} attributable to the error term $E_{l.b.}^\varphi$. Therefore, $u^r$ has the same leading order term as $v_{a^r}$, i.e., $\lim_{T \rightarrow - \infty} \frac{1}{\sqrt {|T|}} (u^r(x,T) - v_{a^r}(x,T)) =0$.   

We can use the same potential $u^r$ to analyze the so-called \textit {comb adversary} $a^c$, which is defined via \textit{ranked coordinates} $\{(i)\}_{i \in [N]}$ such that $x_{(1)} \geq x_{(2)} \geq ... \geq x_{(N)}$.
\begin{tcolorbox}
 At each $t$, \textit{the comb adversary} $a^c$ assigns probability $\frac {1} {2}$ to each of $q^c $ and $-q^c$ where $q^c_{(i)} = 1$ if $i$ is odd and $q^c_{(i)} = -1$ if $i$ is even.
\end{tcolorbox}
In Appendix \ref{app:comb_adversary}, we show that  $ \langle D^2\varphi \cdot q^c,q^c\rangle \geq \Delta \varphi^r$. Therefore, $u^r$  combined with the adversary $a^c$ also satisfies \eqref{eq:pde_fixed_lb}.  \citet{gravin16} conjectured that $a^c$ might be optimal asymptotically in $T$ for any fixed $N$ and \citet{abbasi} and \citet{bayraktar2019b} showed that to be the case for $N=3$ and $4$, respectively.    We do not resolve this conjecture for general $N$, and since  \eqref{eq:pde_ineq_lb} is not satisfied with an equality, our analysis does not guarantee that $u^r$ has the same leading order term as $v_{a^c}$.  However, our result shows that the $a^c$ is at least as powerful as $a^r$.   The following example summarizes this result and the previous one.

 \begin {example} [Randomized and comb adversaries] 
\label{ex:randomized_lb} Let $u^r$ be the heat potential $\varphi$ with $\kappa = \frac{1}{2}$. Then, the value function $v_{a^r}$ of $a^r$ satisfies the following lower bound: $u^r(x,t) - E_{l.b.}^\varphi(t) \leq v_{a^r}(x,t)$.  Also $u^r$ has the same leading order term in $t$ as $v_{a^r}$.   By equation \eqref{eq:phi_at_final_time}, this bound leads to the  regret bound
$\sqrt { |T|} \mathbb E_G \max G_i - E_{l.b.}^\varphi(T) \leq   \min_p R_T(a^r,p)$.

The same lower bound holds for the value function $v_{a^c}$ of $a^c$ (without a guarantee that $u^r$ matches $v_{a^c}$ at the leading order).  
\end{example}
Since $\lim_{N \rightarrow \infty} \frac{1}{\sqrt {2 \log N}}  \mathbb E \max _i G_i=1$,\footnote {See, e.g, Lemmas A.12  in \citet{cesa-bianchi06}.} we have $\lim_{N \rightarrow \infty} \lim_{T \rightarrow -\infty} \frac{u^r(x,T) - E_{l.b.}^\varphi(T)} {\sqrt{2 |T| \log N}} = 1$. 
Thus, in the limit where $T \rightarrow -\infty$ first, and then $N \rightarrow \infty$, the value function $v_{a^c}$ of the comb adversary $a^c$ matches the upper bound given by the exponential weights player $p^e$. Therefore, this adversary is doubly asymptotically optimal (previously this was only known for  $a^r$).

Next, we introduce a new adversary $a^h$ (\textit{heat adversary}). 
\begin{tcolorbox}
 At each $t$, \textit{the heat adversary} $a^h$  samples $q_t$ uniformly from the following set $S$: 
\begin{align*}
S=    \Big \{q\in\{\pm 1 \}^N\mid \sum_{i \in [N]}q_i=\pm 1 \Big \} ~\text{if}~ N ~\text{is odd}~~ \text{or}~~  \Big\{q\in\{\pm 1\}^N\mid \sum_{i \in [N]} q_i=0  \Big\} ~\text{if}~ N~\text{is even}.
\end{align*}

 \end{tcolorbox}
This adversary is symmetric because it is the uniform distribution over the symmetric set $S$. In Appendix \ref{app:heat_lb_kappa}, we show that $\kappa_h \Delta \varphi =   \frac {1}{2} \mathbb E_{a^h}  \langle D^2 \varphi \cdot q, q \rangle$ for
\begin {align}
\kappa_h  = 1 ~~\text{if}~ N =2, ~~~~ \frac{1}{2} + \frac {1}{2N}~~ \text{if}~ N ~\text{ is odd}, ~~~\text {or}
~~~\frac{1}{2} + \frac {1}{2N-2} ~~\text{otherwise.}
\label{eq:heat_lb_kappa}
\end{align}

The potential $u^h$ given by $\varphi$ with the diffusion factor $\kappa = \kappa_h$, combined with the adversary $a^h$, satisfies \eqref{eq:pde_fixed_lb}. Also both inequalities in  \eqref{eq:pde_fixed_lb} are satisfied with equalities, and therefore, $u^h$ has the same leading order term in $t$ as $v_{a^h}$. The resulting lower bound is described in Example \ref{ex:heat_bounds}. 

Similar ideas are used to give an upper bound. In Appendix \ref{app:heat_ub}, we show that
$ \frac {1} {2}  \max_{q \in [-1,1]^N}  \langle D^2 \varphi \cdot q, q \rangle \leq    \Delta \varphi$. Also in Appendix \ref{app:heat_derivs}, we prove $\partial_{i} \varphi \geq 0$  for all $i \in [N]$. Thus, $w^h$ given by $\varphi$ with $\kappa = 1$  satisfies \eqref{eq:pde_fixed_ub} and is associated with the following strategy. 
 \begin{tcolorbox}
  At each $t<-1$, \textit{the heat player} $p^h$ selects $p^h_t = \nabla w^h(x,t+1)$ and, at $t = -1$, the player selects an arbitrary distribution in $ \Delta_N$. 
 \end{tcolorbox}
\begin {example} [New heat-based strategies]
\label{ex:heat_bounds} 
The value function $v_{a^h}$  of $a^h$ satisfies the lower bound  $u^h(x,t) - E_{l.b.}^\varphi(t) \leq v_{a^h}(x,t)$, and the value function $v_{p^h}$  of $p^h$ satisfies the upper bound  $v_{p^h}(x,t) \leq w^h(x,t) + E_{u.b.}^\varphi(t)$, where $u^h$ and $w^h$ are the potentials given above. Also $u^h$ has the same leading order term in $t$ as $v_{a^h}$. Using equation \eqref{eq:phi_at_final_time}, these bounds lead to the regret bounds $\sqrt {2\kappa_h |T|} \mathbb E_G \max G_i - E_{l.b.}^\varphi(T) \leq   \min_p R_T(a^h,p)$ and $\max_a R_T(a,p^h) \leq \sqrt {2 |T|} \mathbb E_G \max G_i + E_{u.b.}^\varphi(T)$. 
\end{example}
For two experts, the lower and upper bounds in the Example above have a matching leading order term $\sqrt {\frac {2}{\pi} |T|}$.  Therefore, the corresponding strategies are minmax optimal asymptotically in $T$.

\section{Max Potentials}
\label{sec:max_fixed}

In this section, we consider the \textit{max potential} $\psi$ given by the solution of:
\begin{align}
\label{eq:max_pde}
\begin{cases}
\psi_t + \kappa \max_{i}  \partial_i^2 \psi =0  \\
\psi(x,0) = \max_i x_i
\end{cases}
\end {align}
 \citet{abbasi}, using random walk methods, showed that an adversary $a^m$ associated with $\psi$ (the \textit {max adversary}) is asymptotically in $T$ optimal for $N=3$.
 \begin{tcolorbox}
 At each $t$, \textit {the max adversary} $a^m$ assigns equal probability to  $q^m $ and $-q^m$ where the entry of $q^m$ corresponding to the largest component of $x$ is  1 and the remaining entries are $-1$.
 \end{tcolorbox} 
 
  There is an explicit formula for $\psi$. Its building blocks are functions of the form $g(x,t) = \sqrt {-2 \kappa t} f \big (\frac {x}{\sqrt{-2 \kappa t}} \big)$ where
\begin{align}
\label{eq:sturm_ode_soln}
f(z) = \sqrt{\frac {2}{ \pi}} e^{-\frac{z^2}{2}} +z \text{erf} \left ( \frac{z}{\sqrt{2}} \right) ~~&\text{and} ~~\text{erf}(y)  = \frac {2}{\sqrt \pi }  \int_0^{y } e^{-s^2} ds.
\end{align}
As shown in Appendix \ref{app:max_pde}, $f$ solves %\footnote {Note this equation is equivalent to $\partial_z (e^{z^2/2} f'(z)) - e^{z^2/2} f(z) = 0$ with the same boundary condition, which is a Sturm-Liouville linear ordinary differential equation.}
$f(z) =  f''(z) + z f'(z)$ with  $\lim_{|z| \rightarrow \infty} \frac {f(z)}{ |z|}=1$. Therefore, $g(x,t)$ solves the 1D linear heat equation on $\mathbb R \times \mathbb R_{<0}$:
$ g_t + \kappa g_{xx} =0$ with $g(x,0) =|x|$.  We define $\psi$ globally in a uniform manner using ranked coordinates given in Section \ref{sec:heat_fixed}, and verify the following Claim in Appendix \ref{app:max_pde}.  
\begin{claim} 
\label{cl:max_pde_soln}
Equation \eqref{eq:max_pde} has an explicit classical solution on $\mathbb R^N \times \mathbb R_{<0}$, namely
 \begin{align} 
 \label{eq:max_pde_soln}
\psi(x,t) &= \frac {1}{N} \sum_i x_{(i)} + \sqrt{-2\kappa t } \sum_{l=1}^{N-1} c_l f(z_l) 
\end{align}
where $ z_l = \frac {1}{\sqrt {-2\kappa t}} \left ( \left (\sum_{n=1}^l{x_{(n)}} \right) - lx_{(l+1)} \right)$, $f$ is given by \eqref{eq:sturm_ode_soln} and $c_l = \frac{1}{l(l+1)}$.
\end{claim}
Since $z_l$ does not change when a multiple of $\mathbb 1$ is added to $x$, we have $\psi(x+ c \mathbb 1, t)=\psi(x, t)+c$.  Therefore, $\psi$ satisfies \eqref{eq:pde_boundary_lb} and \eqref{eq:pde_boundary_ub}.

Appendix \ref{app:max_lb} shows that  $\langle D^2 \psi \cdot q^m,  q^m \rangle  =4 \max_j \partial_{jj} \psi$. Therefore, $u^m$ given by $\psi$ with $\kappa = 2$ satisfies \eqref{eq:pde_ineq_lb} for the adversary $a^m$.  Also both inequalities in \eqref{eq:pde_fixed_lb}  are satisfied with equalities.  Therefore, similarly to the discussion of $u^r$ and $u^h$ in Section \ref{sec:heat_fixed}, $u^m$ has the same leading order term as $v_{a^m}$. The resulting lower bound is given in Example \ref{ex:max_bound}. 
 
To determine an upper bound, in Appendix \ref{app:max_kappa}, we prove that $\frac{1}{2}\max_{q \in [-1,1]^N} \langle D^2\psi \cdot q, q \rangle  \leq \kappa_m \max_i \partial_i^2 \psi$ for   
\begin{align}
\kappa_m= \frac{N^2}{2(N-1)} ~\text{if}~ N~\text{is even} ~~~\text{or} ~~~ \frac{N+1}{2} ~ \text{if}~ N~\text{is odd}
\label{eq:lower_bound_max_factor}
\end{align}
Also in Appendix \ref{app:Derivatives of the max-based potential} we show $\partial_{i} \psi \geq 0$ for all $i$ in $[N]$. Therefore, an upper bound potential $w^{m}$ given by $\psi$ with $\kappa =\kappa_m$ satisfies \eqref{eq:pde_fixed_ub} and is associated with the following strategy (\textit{max player}).
 \begin{tcolorbox}
 The \textit{max player} $p^m$  selects $p^m_t = \nabla w^m(x,t+1)$ at  $t < -1$ and   an arbitrary $p_{-1} \in  \Delta_N$ at $t = -1$.
 \end{tcolorbox}
 
Since the formula \eqref{eq:max_pde_soln} for $\psi$ uses ranked coordinates, particular scrutiny is needed on the boundaries where the ranking changes. The calculation in Appendix \ref{app:maxC2}  reveals that the third-order spatial derivatives do not exist on those boundaries.
%Since $\psi$ is constructed by reflection of a smooth function whose first derivatives normal to the reflection boundary vanish, its third spatial derivatives are bounded almost everywhere on $\mathbb R^N$ but are discontinuous at the reflection boundary. 
Therefore, Proposition \ref{rem:fixed_lb_lipschitz} is not available in this setting.

Let $E_{l.b.}^\psi$ denote the error term within the meaning of Theorem \ref{thm:fixed_lb} for $\psi$ with  $\kappa=2$ and  the associated adversary $a^m$. Appendix \ref{app:lb_err_max} shows that  $E_{l.b.}^\psi(t)= O( N  \log |t|)$. Let $E_{u.b.}^\psi$ denote the ``error" term within the meaning of Theorem \ref{thm:fixed_ub} for $\psi$ with $\kappa=\kappa_m$. Appendix \ref{app:ub_err_max} shows that $E_{u.b.}^\psi(t) =O( N  \log |t|)$ as well.\footnote{While the asymptotic notation is used here for conciseness as well, the Appendix provides explicit error bounds.}

 \begin {example} [Max-based strategies] 
\label{ex:max_bound}
The value function $v_{a^m}$ of $a^m$ satisfies the lower bound $u^m(x,t) - E_{l.b.}^\psi(t) \leq v_{a^m}(x,t)$   and  the value function  $v_{p^m}$ of $p^m$ satisfies the  upper bound  $v_{p^m}(x,t) \leq w^m(x,t)  + E_{u.b.}^\psi(t)$,  where $u^m$ and $w^m$ are the potentials defined above. Also $u^m$ has the same leading order term in $t$ as $v_{a^m}$. Since $\psi(0,T) =\frac {2(N-1)}{N}  \sqrt {\frac {\kappa}{\pi} |T|} $, the regret satisfies the  bounds $ \frac {2(N-1)}{N}  \sqrt {\frac {2}{\pi} |T|}  - E_{l.b.}^\psi(T) \leq  \min_p R_T(a^m,p)$ and $\max_a R_T(a,p^m) \leq  \frac {2(N-1)}{N}\sqrt {\frac {\kappa_m}{\pi} |T|}  + E_{u.b.}^\psi(T)$. 
\end{example}
The lower and upper bounds have the matching leading order term of  $\sqrt {\frac {2}{\pi} |T|}$ and $4\sqrt {\frac {2}{9 \pi} |T|}$ for, respectively, two and three experts. Therefore, the corresponding strategies are minmax optimal asymptotically in $T$.  The same leading order constant for three experts was determined in \citet{abbasi} (after rescaling for our $[-1,1]^N$ loss function) with an $O(\log^2 |T|)$ error term.   Our method, however, reduces the error to $O(\log |T|)$.
 
\section {Related Work} 

In this Section, we first describe the relationship of our potentials to the PDE characterizing minimax optimal value. Second, we compare our bounds with the previously known ones.
 
\subsection {PDE Characterizing Minimax Optimal Value} 
\label{sec:relationship_to_dk}

The fact that our bounds for $N =2, 3$ match asymptotically can be understood from a PDE perspective. Indeed, our upper and lower-bound heat and max potentials for $N=2$ are the same. Our upper and lower-bound max potentials for $N=3$ are the same as well.  They all solve the PDE derived as in \cite{drenska2019prediction} that, as noted earlier, characterizes the asymptotically optimal result. This observation can also be found in \cite{bayraktar2019b} (for $N=4$, however, the solution of the relevant PDE is different from our potentials).

 \citet{drenska2019prediction} showed that, for any fixed $N$, the leading order term of the minimax value  function is the unique viscosity solution of the associated nonlinear PDE. Although the $\{0,1\}^N$ adversary in that reference is different from our $[-1,1]^N$ adversary, this is not consequential. Thus, the relevant PDE, as adjusted for our adversary, is the following:
\begin{align}
\label{eq:drenska_pde}
\begin{cases}
v_t + \frac{1}{2} \max_{q \in [-1,1]^N} \langle D^2 v \cdot q, q \rangle =0 \\
v(x,0) = \max_i x_i
\end{cases}
\end{align}
Since for an arbitrary $N$ the solution $v$ is not known explicitly, the PDE \eqref{eq:drenska_pde} does not provide a numerical estimate of the regret; moreover it only describes the leading order behavior as $|T| \rightarrow \infty$.  Our framework, by contrast, gives explicit upper and lower bounds, which hold for any $T$. 

While our framework does not use the PDE \eqref{eq:drenska_pde}, it is not unrelated.  Indeed, since a lower bound potential $u$ must satisfy \eqref{eq:pde_fixed_lb}, it has $\max_{q\in[-1,1]^N} \langle D^2u \cdot q, q \rangle \geq  \mathbb E_{q\sim a_t} \langle D^2u \cdot q, q \rangle$. Therefore, $u$ is a so-called \textit{subsolution} of \eqref{eq:drenska_pde}. Since these PDEs have a comparison principle,   $u \leq v$.  Similarly, an upper bound potential $w$ given by a solution of \eqref{eq:pde_fixed_ub} is a \textit{supersolution} of \eqref{eq:drenska_pde}, which implies $v \leq w$.

While the preceding remarks provide insight about why our potentials work, they rely upon the comparison principle for viscosity solutions of  \eqref{eq:drenska_pde} -- a result which is by no means elementary.    Our arguments (which build on the insight of \citet{rokhlin}) are, by contrast, entirely elementary, using little more than Taylor expansion.   (Our overall framework, presented in Appendices \ref{app:thm_fixed_lb} and \ref{app:thm_fixed_ub}, resembles a ``verification argument" from  optimal control theory.)

\subsection {Relationship to Existing Bounds}
\label{sec:related_work}
\begin{figure}
  \centering
    {\includegraphics[width=0.49\linewidth]{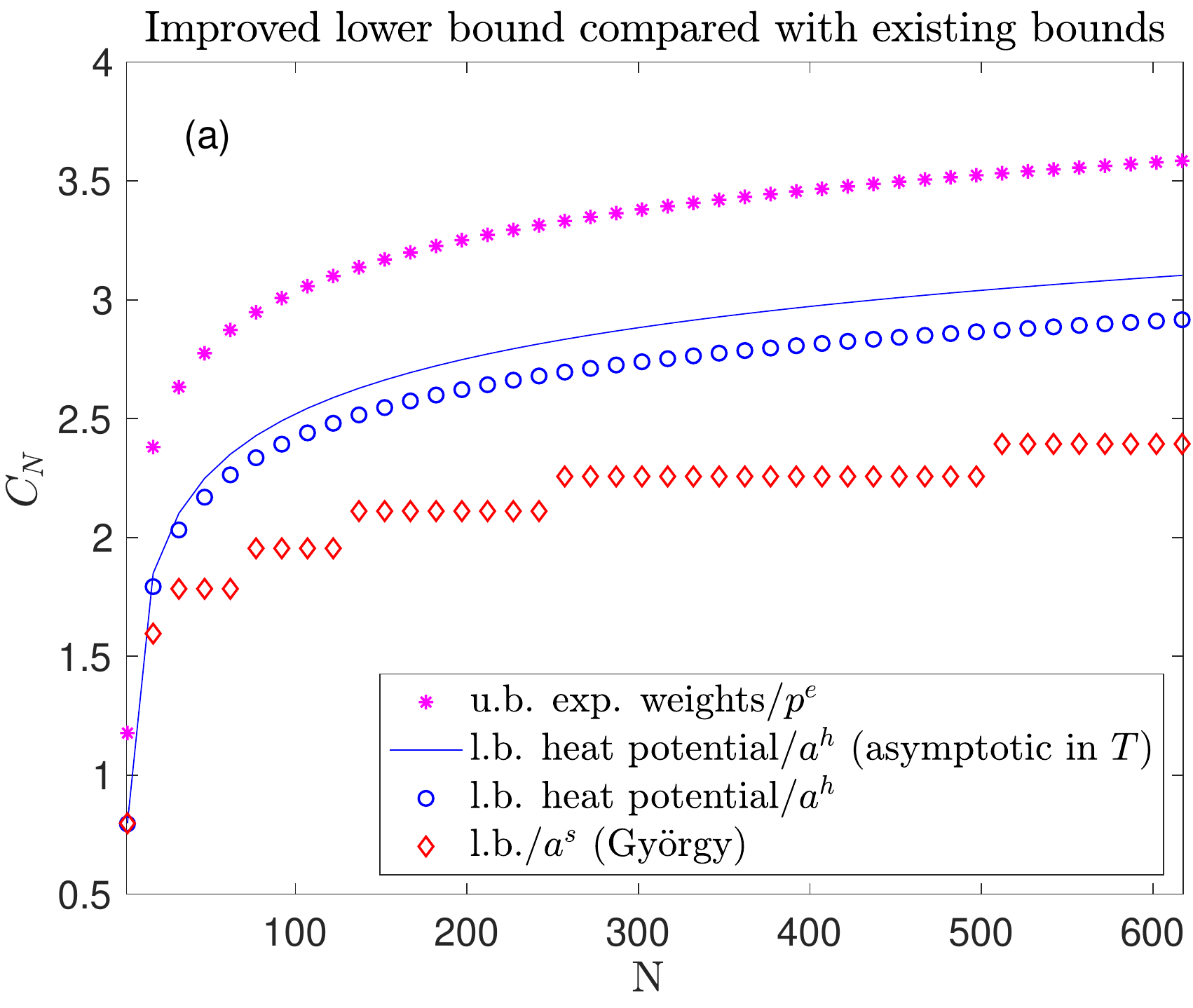}
   \phantomsubcaption\label{fig:l}}
  {\includegraphics[width=0.49\linewidth]{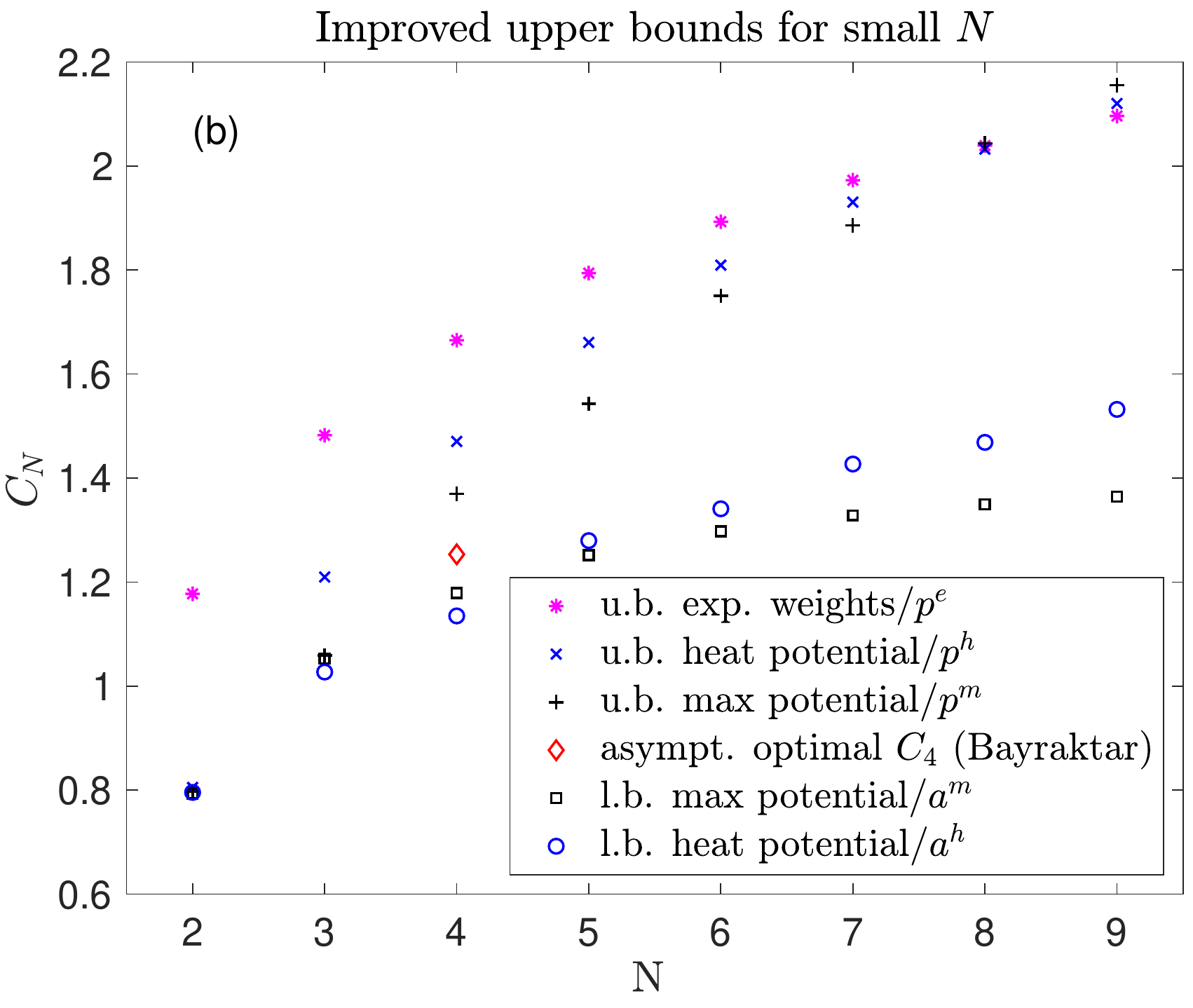}
   \phantomsubcaption\label{fig:2}}
  \caption{ Plots of $C_N$ with $N$ where $C_N = (u(0, T) - E(T))/\sqrt {|T|}$ for a lower bound (l.b.) potential $u$ and the associated adversary $a$  (the resulting l.b. is $C_N \sqrt {|T|} \leq \min_p R_T(a,p)$) and $C_N$  equal to $(w(0, T) + E(T))/\sqrt {|T|}$ for an upper bound (u.b.) potential $w$ and the corresponding player $p$ (the resulting u.b. is  $\max_a R_T(a,p) \leq C_N \sqrt {|T|}$). Each $C_N$ is determined for $T = -10^7$ except where it is specified to be asymptotic in $T$. Plot \subref{fig:l} compares the improved l.b. with previously known l.b. and u.b's and plot \subref{fig:2} shows the improved u.b. for small $N$ (the exponential weights u.b. various l.b.'s and the asymptotically optimal $C_4$ are also plotted for reference).}
  \label{fig:CN}
\end{figure}

Note that $\kappa_h$ is strictly larger than $\frac {1}{2}$ for any given $N$.  Therefore, asymptotically in $T$, the lower bound attained by our heat adversary $a^h$ is tighter than the one attained by the classic randomized adversary $a^r$.

When $N$ and $T$ are fixed, a bound obtained using $a^r$ is provided by  in \citet{orabona}; their argument involves lower bounding the maximum of $N$ independent symmetric random walks of length $|T|$.  Another lower bound is given in Chapter 7 of \cite{gyorgy} for an adversary $a^s$ constructed from a single random walk of length $|T|$.   This $a^s$ provides a tighter lower bound than our $a^h$ when $|T|$ is relatively small. However, as illustrated by Figure \ref{fig:CN}\subref{fig:l}, when $|T|$ is large, our strategy $a^h$ improves on the lower bound obtained using $a^s$. (The lower bound given by \citet{orabona} is not shown because its value is negative for the given $T$ and range of $N$.)   

Turning to the upper bounds: when $N$ is small and $|T|$ is large, as illustrated by Figure \ref{fig:CN}\subref{fig:2}, the max player $p^m$ improves on the upper bound given by the exponential weights $p^e$. (The heat player $p^h$  also improves on $p^e$ in this setting.) See Appendix \ref{app:numerics} for details regarding the numerical computation of these bounds. 

\section{Conclusions}

We establish that potentials can be used to design effective strategies leading to lower bounds as well as upper bounds. We also provide a scheme by which solutions of well-chosen PDEs can be used as upper bound or lower bound potentials. The resulting bounds improve in some cases upon the previously known bounds. 

While this paper focuses on the fixed horizon version of the expert problem, \cite{kobzar_geom} extends our framework to the \textit{geometric stopping} version, where the final time is not fixed but is rather random, chosen from the geometric distribution.  

% Acknowledgments---Will not appear in anonymized version
\acks{V.A.K and R.V.K. are supported, in part, by NSF grant DMS-1311833. V.A.K. is also supported by the Moore-Sloan Data Science Environment at New York University.}

\bibliography{expert_bounds_article}

\appendix

\section{Proof of Theorem \ref{thm:fixed_lb}}
\label{app:thm_fixed_lb}

Since $v_a$ is characterized by the dynamic program \eqref{eq:va_dp}, we show  that $u(x,t)- E(t) \leq v_a(x,t)$ by induction starting from the final time.  The initial step follows from the inequality between $v_a$ and $u$ at $t=0$.  To prove the inductive step, as a preliminary result, we bound below  $\min_{p} \mathbb E_{a_t, p} ~[u  ( x + r,  t +1)]- u  ( x ,  t ) $ in terms of $C$ and $K(t)$.  At $t=-1$, the conditions of the theorem already provide: 
\begin{align*} 
 \min_{p} ~ \mathbb E_{a_{-1}, p} ~[u  ( x + r,  0)]  - u(x,-1) \geq - C
\end{align*}
For $t \leq -2$, we note that $r = q_I \mathbb 1 -q$ and use the linearity of $u$ in the direction of  $\mathbb 1$: 
\begin{align*}
&\min_{p} ~ \mathbb E_{p,a_t} ~[u  ( x + r,  t +1)] - u(x,t)  \\
&=\min_{p} ~ \mathbb E_{p,a_t} ~[u  ( x - q,  t +1) + q_I]  - u  ( x ,  t +1) +  u  ( x ,  t+1 ) - u(x,t) 
 \end{align*}
Since $u(\cdot,t+1)$ is $C^2$ with Lipschitz continuous second-order derivatives in $x$, we use Taylor's theorem with the integral remainder 
\begin{align}
\label{eq:spatial_taylor_mv}
u(x-q,t+1)= &u(x,t+1)-\nabla u(x,t+1)\cdot q+\frac{1}{2}\langle D^2 u(x,t+1)\cdot q,q\rangle \nonumber\\
&-\int_0^1 D^3u(x-\mu q,t+1)[q,q,q]\frac{(1-\mu)^2}{2}d\mu
\end{align}
Thus,
\begin{align}
\label{eq:spatial_taylor_bound}
u(x-q,t+1)-u(x,t+1) +q_I \geq&q_I -\nabla u(x,t+1)\cdot q+\frac{1}{2}\langle D^2 u(x,t+1)\cdot q,q\rangle \nonumber\\
&- \frac {1}{6} \text{ess sup}_{y  \in [x, x-q]}  D^3u( y, t+1) [q,q,q] 
\end{align}
Similarly, $
u  ( x , t+1 ) - u(x,t)\geq  u_{t}(x, t+1 ) - \frac {1}{2}  \text{ess sup}_{\bar t \in [t, t+1]} u_{tt}(x,\bar t)$.

The rules of the game provide that $q$ distributed according to $a_t$ and $I$ distributed according to $p_t$ are independent conditioned on history. Therefore, $\mathbb E_{p,a_t}  [q_I - \nabla u(x,t+1)\cdot q] = \langle p - \nabla u(x,t+1), \mathbb E_{a_t} q \rangle = 0$ for all $p$ since $a_t$ is balanced and  $\sum_{i} \partial_i u =1$ by linearity of $u$ along $\mathbb 1$. As a result, we can eliminate the dependence on $p$.  Also we note the condition on the potential \eqref{eq:pde_ineq_lb}.

Using the foregoing results and the definition of $K$, we obtain
\begin{align}
\min_{p} ~ \mathbb E_{p,a_t} ~[u  ( x + r,  t +1)] - u(x,t)\geq -K(t)=E(t+1)-E(t) 
\label{eq:lb_error}
\end{align}

Finally, using \eqref{eq:lb_error}, the inductive hypothesis $u(x+ r ,t+1) -E(t+1) \leq v_a(x+ r ,t+1)$, and the dynamic program formulation of $v_a$ in \eqref{eq:va_dp}, we obtain
\begin{align*}
 u  ( x, t)  - E(t) &\leq  u(x, t) +  \min_{p} \mathbb E_{p,a_t} ~u  ( x + r,  t +1) - u(x,t) - E(t+1) \\
& \leq \min_{p}  ~ \mathbb E_{p,a_t}  [v_a(x+ r,t+1)] =v_a ( x,  t) 
\end{align*}

\section{Proof of Proposition \ref{rem:fixed_lb_lipschitz}}
\label{app:fixed_lb_lipschitz}
If $D^3 u(\cdot,t+1)$ exists and is Lipschitz continuous, then \eqref{eq:spatial_taylor_mv} can be replaced by
\begin{align*}
u(x-q,t+1)=&u(x,t+1)-\nabla u(x,t+1)\cdot q+\frac{1}{2}\langle D^2 u(x,t+1)\cdot q,q\rangle \\
&-\frac{1}{6} D^3 u(x,t+1) [q,q,q]+\int_0^1 D^4u(x-\mu q,t+1)[q,q,q,q]\frac{(1-\mu)^3}{6}d\mu
\end{align*}
and in such case \eqref{eq:spatial_taylor_bound} is replaced by 
\begin{align*}
u(x-q,t+1)-u(x,t+&1) +q_I \geq q_I -\nabla u(x,t+1)\cdot q+\frac{1}{2}\langle D^2 u(x,t+1)\cdot q,q\rangle \\
&-\frac{1}{6} D^3 u(x,t+1) [q,q,q]+ \frac {1}{24} \text{ess sup}_{y  \in [x, x-q]}  D^4u( y, t+1) [q,q,q,q] 
\end{align*}
Since the adversary $a$ is symmetric, $q$ has the same distribution as $-q $. Therefore, $\mathbb E_{a_t}q_iq_jq_k =- \mathbb E_{a_t} q_iq_jq_k$, for any $i$, $j$, and $k$. This implies $\mathbb E_{a_t}q_iq_jq_k =0$ and consequently $\mathbb E_{a_t}D^3u(x,t+1)[q,q,q]=0$.  The remainder of the proof of Theorem \ref{thm:fixed_lb} is the same except that we use the definition of $K$ given in this Proposition.

\section{Proof of Theorem \ref{thm:fixed_ub}}
\label{app:thm_fixed_ub}

Since $v_p$ is characterized by the dynamic program \eqref{eq:vp_dp}, we show by induction that $v_p(x,t) \leq w(x,t)+ E(t) $.   The initial step follows from the inequality between $v_p$ and $w$ at $t=0$, and the rest of the proof is similar to the oroof of Theorem \ref{thm:fixed_lb}.  To prove the inductive step, we note that $\max_{a} ~ \mathbb E_{ p_{-1}, a} ~[w  ( x + r,  0)]  - w(x,-1) \leq C$. For $t \leq -2$, we again note that $r = q_I \mathbb 1 -q$ and use the linearity of $w$ in the direction of  $\mathbb 1$: 
\begin{align}
&\max_{a} ~ \mathbb E_{p_t,a} ~[w  ( x + r,  t +1)] - w(x,t) \nonumber \\
&=\max_{a} ~ \mathbb E_{a_t} ~[w  ( x - q,  t +1)+p_t\cdot q ] - w  ( x ,  t +1) +  w  ( x ,  t+1 ) - w(x,t)
\label{eq:diff_1}
 \end{align}
The equality above also uses the fact that under the rules of the game,  $q$ distributed according to $a_t$ and $I$ distributed according to $p_t$ are independent, conditionally on history.
Since $w(\cdot,t+1)$ is $C^2$ with Lipschitz continuous second order derivatives, we again use Taylor's theorem with the integral remainder 
\begin{align}
w(x-q,t+1)=&w(x,t+1)-\nabla w(x,t+1)\cdot q+\frac{1}{2}\langle D^2 w(x,t+1)\cdot q,q\rangle   \nonumber\\
&-\int_0^1 D^3w(x-\mu q,t+1)[q,q,q]\frac{(1-\mu)^2}{2}d\mu \label{eq:diff_2}
\end{align}
The fact that $p_t=\nabla w(x,t+1)$ provides that $p_t\cdot q-\nabla w(x,t+1)\cdot q =0$ for all $q$.  Thus
\begin{align*}
w(x-q,t+1)+ p_t\cdot q - w(x,t+1)\leq & \frac{1}{2}\langle D^2 w(x,t+1)\cdot q,q\rangle\\
&- \frac {1}{6} \text{ess inf}_{y  \in [x, x-q]}  D^3w( y, t+1) [q,q,q] 
\end{align*}
Similarly, 
\begin{align}
w  ( x , t+1 ) - w (x,t)\leq  w_{t}(x, t+1 ) - \frac {1}{2}  \text{ess inf}_{\tau \in [t, t+1]} w_{tt}(x,\tau)
 \label{eq:diff_3}
\end{align}
Also we note the following condition on the potential \eqref{eq:pde_ineq_ub}. 
By collecting the above inequalities and using the definition of $K$, 
\begin{align}
\max_{a} ~ \mathbb E_{p_t,a} ~[w  ( x + r,  t +1)] - w(x,t)\leq K(t)=E(t)-E(t+1) 
 \label{eq:diff_4}
\end{align}

Using the inequality \eqref{eq:diff_4}, the inductive hypothesis $w(x+ r ,t+1) +E(t+1) \geq v_p(x+ r ,t+1)$, and the dynamic program formulation of $v_p$ in \eqref{eq:vp_dp}, we obtain
\begin{align*}
 w  ( x, t)  + E(t) &\geq  w(x, t) +  \max_{a} \mathbb E_{p_t,a} ~w  ( x + r,  t +1) - w(x,t) + E(t+1) \\
& \geq \max_{a}  ~ \mathbb E_{p_t,a}  [v_p(x+ r,t+1)] =v_p ( x,  t) 
\end{align*}

\section{Proof of Proposition \ref{rem:constant_w_t}}
\label{app:constant_w_t}
By  definition, $w$ is twice differentiable in $x$ for all $x$ and $t < 0$. Then, the form $w(x,t) =  \Phi(x)+ct$,  implies that $w$ is so differentiable for all $t$. Therefore, we bound \eqref{eq:diff_1} using a Taylor expansion starting at $t=-1$, rather than $t=-2$. In this case, it suffices to show that $K(t) = 0$ for all $T \leq t \leq -1$. Noting that  $D^2 w(x,t) =D^2 \Phi(x)$, we use Taylor's theorem with the mean value form of the second-order (in $x$) remainder. Thus, \eqref{eq:diff_2} is replaced by 
\[
w(x-q,t+1)=w(x,t+1)-\nabla w(x,t+1)\cdot q+\frac{1}{2}\langle D^2 \Phi(y)\cdot q,q\rangle
\]
for $y = x- \mu q$ and some $\mu \in [0,1]$.  Since $w_t$ is constant, \eqref {eq:diff_3} is replaced by $w  ( x , t+1 ) - w (x,t)= w_t = c$. Therefore, \eqref{eq:diff_4} is replaced by $\max_{a} ~ \mathbb E_{p_t,a} ~[w  ( x + r,  t +1)] - w(x,t)\leq 0$.  The rest of the proof of Theorem \ref{thm:fixed_ub} is the same; it reveals that  $w  ( x, t) \leq   v_p ( x,  t)$ for all $T \leq t \leq -1$ and all $x$, as desired.
\section{Hessian of the Exponential Weights Potential}
\label{app:exp_ub}
 By a standard result, $\Phi$ is convex.\footnote{See, e.g., Sec. 3.1.5 in  \cite{boyd_cvx_book}.} Therefore, $D^2 \Phi$ is a positive semidefinite matrix, and its quadratic form $ \langle D^2 \Phi \cdot q, q \rangle$ is maximized at one of the extreme points $\{ \pm 1\}^N$. 
 Note that 
\begin{align*}
\partial_{ij} \Phi(x, t)= 
\begin{cases}
\psi''(y)\phi'(x_i) \phi'(x_j) &\text {if}~  i \neq j \\
\psi''(y)\phi'(x_i)^2+ \psi'(y) \phi''(x_i) &\text {if}~  i = j \\
\end{cases} 
\end{align*}
where
\begin{align*}
&y = \sum_{k=1}^N \phi(x_k),~ \psi(y) = \frac{1}{\eta} \log(y), ~\psi'(y) = \frac{1}{\eta y}, ~ \psi''(y) = -\frac{1}{\eta y^2},\\
 &\phi(x_k) =  e^{\eta x_k}, ~ \phi'(x_k) =  \eta e^{\eta x_k} ~ \text{and}~ \phi''(x_k) =  \eta^2 e^{\eta x_k}
\end{align*}
%Observe that $D^2\Phi (x+a\mathbb 1) = D^2\Phi (x)$, and  $\sum_j \partial_{ij} \Phi = 0$ (and thus, $D^2\Phi \cdot \mathbb 1=0$) by linearity of $w$ in the direction of $\mathbb 1$. 
Using these results, for all $q \in \{\pm\}^N$
\begin{align}
\langle D^2\Phi \cdot q, q \rangle &= -\eta \left (\sum_{k=1}^N e^{\eta x_k} \right)^{-2} \sum_{i,j} e^{\eta x_i} e^{\eta x_j} q_i  q_j + \eta = \eta - \eta \langle p^e, q \rangle^2  \leq  \eta \nonumber \label {eq:second_deriv_bound} 
\end{align}

\section{Heat Potential Error Terms}
\label{app:heat_bounds} 

In this Appendix, we compute the error terms for the heat potential $\varphi$ given by \eqref{eq:heat_potential}.  As a preliminary result, in Appendix \ref{app:heat_derivs}, we compute the spatial derivatives of $\varphi$ up to the 4th order and determine their sign. In Appendix \ref{app:lb_err}, we determine the lower bound error term $E^{\varphi}_{l.b.}$  for an arbitrary adversary supported on $\{ \pm 1\}^N$. Since $\varphi$ is smooth, we use Proposition \ref{rem:fixed_lb_lipschitz} for purposes of the lower bound.   Finally, in Appendix \ref{app:ub_err}, we  determine the upper bound error term $E^{\varphi}_{u.b.}$.

\subsection{Spatial Derivatives of the Heat Potential}
\label{app:heat_derivs}

 Note that $\max _ k (x_k-y_k)$ is differentiable almost everywhere and
\begin{align*}
\partial_i\max_k (x_k-y_k) =  
\begin {cases}
1 ~ \text{if}~ x_i-y_i>\max _ {j\neq i} (x_j-y_j)\\
0 ~\text{if}~ x_i-y_i<\max _ {j\neq i} (x_j-y_j)
\end {cases}
\end{align*}
Therefore, the first derivatives are
\begin{align*}
\partial_i \varphi=\alpha\int e^{-\frac{\|y\|^2}{2
\sigma^2}}\mathbb{1}_{x_i-y_i>\max_{j\neq i}x_j-y_j}dy
= \alpha\int e^{-\frac{\|x-y\|^2}{2
\sigma^2}}\mathbb{1}_{y_i>\max_{j\neq i}y_j}dy\geq 0
\end{align*}
and the second pure derivatives are
\begin{align*}
\partial_{ii} \varphi=&-\frac{\alpha}{\sigma^2}\int e^{-\frac{\|x-y\|^2}{2
\sigma^2}}(x_i-y_i)\mathbb{1}_{y_i>\max_{j\neq i}y_j}dy =-\frac{\alpha}{\sigma^2}\int e^{-\frac{\|y\|^2}{2
\sigma^2}}y_i\mathbb{1}_{x_i-y_i>\max_{j\neq i}x_j-y_j}dy\\
=&-\frac{\alpha}{\sigma^2}\int_{\mathbb{R}^{N-1}}e^{-\frac{\sum_{j\neq i}y_j^2}{2
\sigma^2}}\int_{-\infty}^{x_i-\max_{j\neq i}x_j-y_j}e^{-\frac{y_i^2}{2\sigma^2}}y_idy_id\hat{y}_i
\end{align*}
where $\hat y_i$ is a vector in $\mathbb R^{N-1}$ containing the same components as $y \in \mathbb R^{N} $ except $y_i$.  Since $\int_{-\infty}^{x_i-\max_{j\neq i}x_j-y_j}e^{-\frac{y_i^2}{2\sigma^2}}y_idy_i<0$, we have $\partial_{ii} \varphi>0$. \\

The second mixed derivatives are
\begin{align*}
\partial_{ij} \varphi=&-\frac{\alpha}{\sigma^2}\int e^{-\frac{\|x-y\|^2}{2
\sigma^2}}(x_j-y_j)\mathbb{1}_{y_i>\max_{k\neq i}y_k}dy=-\frac{\alpha}{\sigma^2}\int e^{-\frac{\|y\|^2}{2
\sigma^2}}y_j\mathbb{1}_{x_i-y_i>\max_{k\neq i}x_k-y_k}dy\\
=&-\frac{\alpha}{\sigma^2}\int_{\mathbb{R}^{N-1}}e^{-\frac{\sum_{k\neq j}y_k^2}{2
\sigma^2}}\mathbb{1}_{x_i-y_i>\max_{k\neq i,j}x_k-y_k}\int_{x_j-x_i+y_i}^{\infty}e^{-\frac{y_j^2}{2
\sigma^2}}y_jdy_jd\hat{y}_j
\end{align*}
Since $\int_{x_j-x_i+y_i}^{\infty}e^{-\frac{y_j^2}{2 \sigma^2}}y_jdy_j>0$, we have $\partial_{ij}\varphi<0$.  \\

The third derivatives are
\begin{align*}
\partial_{iii} \varphi=&-\frac{\alpha}{\sigma^2}\int e^{-\frac{\|x-y\|^2}{2
\sigma^2}} \left (1-\frac{(x_i-y_i)^2}{\sigma^2}\right )\mathbb{1}_{y_i>\max_{j\neq i}y_j}dy\\ 
=&-\frac{\alpha}{\sigma^2}\int e^{-\frac{\|y\|^2}{2
\sigma^2}} \left (1-\frac{y_i^2}{\sigma^2} \right)\mathbb{1}_{x_i-y_i>\max_{j\neq i}x_j-y_j}dy\\
\partial_{ijj} \varphi=&-\frac{\alpha}{\sigma^2}\int e^{-\frac{\|x-y\|^2}{2
\sigma^2}}\left (1-\frac{(x_j-y_j)^2}{\sigma^2} \right)\mathbb{1}_{y_i>\max_{k\neq i}y_k}dy\\
=&-\frac{\alpha}{\sigma^2}\int e^{-\frac{\|y\|^2}{2
\sigma^2}} \left (1-\frac{y_j^2}{\sigma^2} \right)\mathbb{1}_{x_i-y_i>\max_{k\neq i}x_k-y_k}dy
\end{align*}
when $i,j,k$ are all distinct (assuming $N \geq 3$),  
\begin{align*}
&\partial_{ijk} \varphi=\frac{\alpha}{\sigma^4}\int e^{-\frac{\|x-y\|^2}{2
\sigma^2}}(x_j-y_j)(x_k-y_k)\mathbb{1}_{y_i>\max_{l\neq i}y_l}dy\\
&=\frac{\alpha}{\sigma^4}\int e^{-\frac{\|y\|^2}{2
\sigma^2}}y_jy_k\mathbb{1}_{x_i-y_i>\max_{l\neq i}x_l-y_l}dy\\
&=\frac{\alpha}{\sigma^4}\int_{\mathbb{R}^{N-2}}e^{-\frac{\sum_{l\neq j,k}y_l^2}{2
\sigma^2}}\mathbb{1}_{x_i-y_i>\max_{l\neq i,j,k}x_l-y_l}\int_{x_j-x_i+y_i}^{\infty}e^{-\frac{y_j^2}{2
\sigma^2}}y_jdy_j\int_{x_k-x_i+y_i}^{\infty}e^{-\frac{y_k^2}{2
\sigma^2}}y_kdy_kd\hat{y}_{jk}
\end{align*}
Since $\int_{x_j-x_i+y_i}^{\infty}e^{-\frac{y_j^2}{2
\sigma^2}}y_jdy_j\int_{x_k-x_i+y_i}^{\infty}e^{-\frac{y_k^2}{2
\sigma^2}}y_kdy_k>0$, we have $\partial_{ijk} \varphi>0$.\\
where $\hat{y}_{jk}$ is a vector in $\mathbb R^{N-2}$ containing the same components as $y \in \mathbb R^{N} $ except $y_i$ and $y_j$.\\

The fourth derivatives  are
\begin{align*}
\partial_{iiii}\varphi=&\frac{\alpha}{\sigma^4}\int_{\mathbb{R}^N}e^{-\frac{\|x-y\|^2}{2
\sigma^2}}(x_i-y_i)\left (3-\frac{(x_i-y_i)^2}{\sigma^2}\right)\mathbb{1}_{y_i>\max_{j\neq i}y_j}dy\\ 
=&\frac{\alpha}{\sigma^4}\int_{\mathbb{R}^N}e^{-\frac{\|y\|^2}{2
\sigma^2}}y_i\left (3-\frac{y_i^2}{\sigma^2}\right )\mathbb{1}_{x_i-y_i>\max_{j\neq i}x_j-y_j}dy
\end{align*}
\begin{align*}
\partial_{iijj}\varphi=&\frac{\alpha}{\sigma^4}\int_{\mathbb{R}^N}e^{-\frac{\|x-y\|^2}{2
\sigma^2}}(x_i-y_i)\left (1-\frac{(x_j-y_j)^2}{\sigma^2}\right )\mathbb{1}_{y_i>\max_{k\neq i}y_k}dy\\
=&\frac{\alpha}{\sigma^4}\int_{\mathbb{R}^N}e^{-\frac{\|y\|^2}{2
\sigma^2}}y_i\left (1-\frac{y_j^2}{\sigma^2} \right)\mathbb{1}_{x_i-y_i>\max_{k\neq i}x_k-y_k}dy
\end{align*}
\begin{align*}
\partial_{ijii}\varphi=&\frac{\alpha}{\sigma^4}\int_{\mathbb{R}^N}e^{-\frac{\|y\|^2}{2
\sigma^2}}y_i \left (3-\frac{y_i^2}{\sigma^2} \right) \mathbb{1}_{x_j-y_j>\max_{m\neq j}x_k-y_k}dy \end{align*}
\begin{align*}
\partial_{ijjj}\varphi=&\frac{\alpha}{\sigma^4}\int_{\mathbb{R}^N}e^{-\frac{\|y\|^2}{2
\sigma^2}}y_i \left (1-\frac{y_j^2}{\sigma^2}\right )\mathbb{1}_{x_j-y_j>\max_{m\neq j}x_k-y_k}dy \end{align*}
\begin{align*}
\partial_{ijkk}\varphi=&\frac{\alpha}{\sigma^4}\int_{\mathbb{R}^N}e^{-\frac{\|y\|^2}{2
\sigma^2}}y_i \left (1-\frac{y_k^2}{\sigma^2}\right )\mathbb{1}_{x_j-y_j>\max_{m\neq j}x_k-y_k}dy 
\end{align*}
and
\begin{align*}
&\partial_{ijkl}\varphi=-\frac{\alpha}{\sigma^6}\int_{\mathbb{R}^N}e^{-\frac{\|x-y\|^2}{2
\sigma^2}}(x_j-y_j)(x_k-y_k)(x_l-y_l)\mathbb{1}_{y_i>\max_{m\neq i}y_m}dy\\
&=-\frac{\alpha}{\sigma^6}\int_{\mathbb{R}^N}e^{-\frac{\|y\|^2}{2
\sigma^2}}y_jy_k y_l \mathbb{1}_{x_i-y_i>\max_{m\neq i}x_m-y_m}dy \\
&=-\frac{\alpha}{\sigma^6}\int_{\mathbb{R}^{N-3}}e^{-\frac{\sum_{l\neq j,k, l}y_l^2}{2
\sigma^2}}\mathbb{1}_{x_i-y_i>\max_{m\neq i,j,k,l}x_m-y_m} \left ( \prod_{n = \{j,k,l\}} 
  \int_{x_n-x_i+y_i}^{\infty}e^{-\frac{y_n^2}{2
\sigma^2}}y_ndy_n \right) d\hat{y}_{jkl}
\end{align*}
where  $i$, $j$, $k$ and $l$ are all distinct (i.e., assuming $N \geq 4$) and $\hat{y}_{jkl}$ is a vector in $\mathbb R^{N-3}$ containing the same components as $y \in \mathbb R^{N} $ except $y_i, y_j$ and $y_k$. Since $ \int_{x_n-x_i+y_i}^{\infty}e^{-\frac{y_n^2}{2 \sigma^2}}y_ndy_n>0$, we have $\partial_{ijkl}\varphi<0$.

\subsection{Lower Bound Error: Heat Potential}
\label{app:lb_err}

To apply Theorem \ref{thm:fixed_lb} with respect to an adversary supported on $\{\pm 1\}^N$ associated with the heat potential $\varphi$, we determine the error term  $E^{\varphi}_{l.b.}(t) = C+ \sum _{\tau = t}^{-2}  K(\tau)$ where  $C$ is a constant satisfying  $\varphi(x,-1) - \min_{p} \mathbb E_{a_{-1},p} ~\varphi (x + r,  0)  \leq C$ for all $x$, and $K$ is a function satisfying
 \begin{align*}
 \frac {1}{2}\text{ess sup}_{  \bar \tau \in [\tau, \tau+1] }   \varphi_{tt}(x,\bar \tau ) + \frac {1}{6}   \text{ess sup}_{y \in [x, x -q]} ~D^3\varphi(y,\tau+1) [q,q,q]   \leq K(\tau)
  \end{align*}
for all $\tau \in [t,  -2]$, all $q$ in $\{\pm 1\}^N$ and all $x$.

In the remainder of this Appendix \ref{app:heat_bounds}, let $G$ denote an N-dimensional Gaussian random vector with mean $0$ and identity covariance. In Appendix \ref{app:heat_final_period}, we show that $| \varphi(x,-1) - \varphi(x+r,0)|  \leq C$ for all $x$ and $r$ where $C = 2 +   \sqrt  {2 \kappa } \mathbb E \max_ i  G_i$.   The expression $\mathbb E \max _i G_i$ has a closed-form expression for $N \leq 5$.  The asymptotically optimal upper bound for this quantity is  $\sqrt {2 \log N}$ (e.g, Lemmas A.12 and A.13 in \citet{cesa-bianchi06}) and a sharper non-asymptotic upper bound for $N \geq 7$ is provided in \citet{dasgupta14}.  Therefore, $C = O( \sqrt  {\kappa  \log N})$.

In Appendix \ref {app:phi_time_deriv_bound}, we prove that  $|\varphi_{tt}(x,  \tau) | \leq \frac{K_2}{ |\tau|^{\frac{3}{2}}}$ for all $x$ and $\tau \leq - 1$ where $K_2 =   \frac {\sqrt {\kappa}} {2\sqrt{2}} \mathbb{E} \left [ \left |N+2-\|G\|^2 \right |\max_i |G_i| \right]$. To bound $K_2$, we use the fact that  $\mathbb E \left [||G||^{2} \right] = N$, $\mathbb E \left [||G||^{4} \right] = N(N+2)$\footnote{$\mathbb E \left [||G||^{2m} \right] =  { \int_0^{\infty} r^{2m} r^{n-1} e^{-\frac{r^2}{2}}dr} /{\int_0^{\infty}  r^{n-1} e^{-\frac{r^2}{2}}dr}$ can be computed explicitly using properties of the Gamma function.} and $\mathbb E \max_i G_i^2 \leq 2 \log N+ 2 \sqrt {\log N} +1$.\footnote{See, e.g. Example 2.7 in \citet{boucheron13}.} By Cauchy-Schwarz inequality:
\begin{align*}
\mathbb{E} \big[ \left |N+2-\|G\|^2 \right |\max_i |G_i| \big] \leq&\sqrt{\mathbb{E}[(N+2-\|G\|^2)^2]\mathbb{E}[\max_iG_i^2}]\\
\leq&{\sqrt{2(N+2)(2\log N+2\sqrt{\log N}+1)}}
\end{align*}
Therefore, $K_2 = O(\sqrt {\kappa N \log N})$.

In Appendix \ref {app:phi_third_deriv_lb_err}, we show that $| D^3 \varphi [q,q,q](x, t) |  \leq   \frac{1} {|t|} K_3$ for all $q \in[-1,1]^N$ where $K_3 = \frac {1}{\kappa} \left (\frac{3}{\sqrt{2}} \sqrt {N} +  a\mathbb{E} \max_i|1-G_i^2| \right)$ and $a = 1$ for $N =2$ and $2$ for $N \geq 3$. To bound $K_3$, note that % $\mathbb E \left |N-\|Y\|^2 \right|$ is the mean absolute deviation of a chi-square random variable. It can be expressed in closed form in terms of a Gamma function, and is bounded by the standard deviation $\sqrt{2N}$. Finally,
 $\mathbb{E} \max_i|1-G_i^2| \leq \mathbb{E}\max_iG_i^2+1$, where the right-hand side is bounded as described in the preceding paragraph. Therefore, $K_3 = O\left (\frac{\sqrt {N}}{ \kappa}\right)$.

Therefore,  $K(\tau)  = \frac {1}{2} \frac{K_2}{ |\tau+1|^{\frac{3}{2}}}+  \frac {K_3}{6} \frac{1} {|t+1|} $ and 
\begin{align*}
\sum_{\tau = t}^{-2} K(\tau) &= \sum_{\tau = t}^{-2} \frac {1}{2} \frac{K_2}{ |\tau+1|^{\frac{3}{2}}}+  \frac { K_3}{6} \frac{1} {|t+1|} \leq  \sum_{s = 1}^ {|t|-1} \frac{K_2}{2s^{\frac{3}{2}}}+ \frac{K_3} {6s} \leq  \frac{K_2}{2}+ \frac{K_3} {6} +\int_{s = 1}^ {|t|-1} \frac{K_2}{2s^{\frac{3}{2}}}+ \frac{K_3} {6s} ds\\
& =    \frac{K_2}{2} \left (3 - \frac{2}{\sqrt{|t| - 1}}\right) + \frac{K_3} {6} (1+\log(|t| - 1))
 \end{align*}
The foregoing shows that for $\kappa \in [\frac{1}{2}, 1]$, $E^{\varphi}_{l.b.}(t) = O(\sqrt {N \log N}+ \sqrt{N} \log |t|)$  by Theorem  \ref{thm:fixed_lb}. \\

Since $\varphi$ is smooth, Proposition \ref{rem:fixed_lb_lipschitz} is also available: to use it we identify a function $K'$ satisfying  
\begin{align*}
 & \frac {1}{2}\text{ess sup}_{  \bar \tau \in [\tau, \tau+1] }    u_{tt}(x,\bar \tau)   -\frac {1}{24}  \text{ess inf}_{y \in [x, x -q]} ~D^{4}u(y,\tau+1) [q,q,q,q]   \leq K'(\tau)
  \end{align*}
for  all $\tau \in [t, -2]$,  all $q \in \{\pm 1\}^N$ and all $x$. 
In Appendix \ref {app:phi_fourth_deriv_lb_err} we show that for $q \in \{\pm1\}^N$,
$|D^{4}\varphi(x,t) [q,q,q,q]\leq \frac{K_4(t)}{| t|^\frac{3}{2}}$
where $K_4 = \frac{2\sqrt{2}N}{\kappa^\frac{3}{2}}(2\sqrt{6}+3\sqrt{2N+4})$. Therefore,  $K'(\tau)  = \frac {1}{2} \frac{K_2}{ |\tau+1|^{\frac{3}{2}}}+  \frac {1}{24} \frac{K_4}{ |\tau+1|^{\frac{3}{2}}}$ and $ \sum_{\tau = t}^{-2} K'(\tau)  \leq   \left  (\frac{K_2}{2} + \frac{K_4} {24}\right) \left (3 - \frac{2}{\sqrt{|t| - 1}}\right) $. This shows that for $\kappa \in [\frac{1}{2}, 1]$, $E^{\varphi}_{l.b.}(t) = O(N\sqrt {N })$ uniformly in $t$.  Combining this  with the result in the preceding paragraph, we obtain $E^{\varphi}_{l.b.}(t) = O \left( N\sqrt {N }\wedge \sqrt {N \log N}+ \sqrt{N} \log |t| \right)$.

\subsubsection{Bound on $|\varphi(x,-1) - \varphi(x+r,0)|$}
\label{app:heat_final_period} 

%In this setting, we will use the following standard probability result.   Let $G$ be an N-dimensional Gaussian random vector with mean 0 and covariance $I$, and let a change of integration variables be given by $r = \sqrt {2s}$. Then, using the properties of the Gamma function,
%\begin{align}
%\label{eq:gaussian_moments} 
%\mathbb E \left [||G||^{2m} \right] &= \frac { \int_0^{\infty} r^{2m} r^{n-1} e^{-\frac{r^2}{2}}dr}{\int_0^{\infty}  r^{n-1} e^{-\frac{r^2}{2}}dr}=\frac {2^m \int_0^{\infty} s^{m+\frac {N}{2}}  e^{-s}dr}{\int_0^{\infty}  s^{\frac{N}{2}} e^{-s}dr} =2^m\frac {\Gamma (m+\frac {N}{2}) }{\Gamma (\frac{N}{2}) }\\
%&= 2^m  {\frac{\sqrt{\pi}(2m+{N}-2)!!}{2^{\frac{2 m+N-1}{2}} }} \Big/ {\frac{\sqrt{\pi}({N}-2)!!}{2^{\frac{N-1}{2}} }} = N(N+2) ... (N+2m-2)
%\end {align}
%where $n!! = n(n-2)(n-4)...$ denotes the so-called double factorial of $n$.  

% Note that 
%\begin{align*}
%u(0,t) &=\int_{\mathbb R^N} (2 \pi \sigma^2 )^{\frac{N}{2}} e^{\frac  {\|y \|^2} {-2 \sigma^2} } \max_i (y)_i dy = \sigma \mathbb E  \max_i X_i 
%\end{align*}
%where  $X$ is a standard Gaussian random vector $N(0, I_N)$. 

We decompose the difference as follows
\[
\varphi(x+r,0) - \varphi(x,-1)  = \max_i(x+r)_i -  \max_i x_i +\varphi(x,0) - \varphi(x,-1) 
\]
Since $r=q_I \mathbb 1-q \in[-2,2]^N$, we obtain $-2\leq \max_i (x+r)_i -  \max_i x_i\leq 2$. Also since $ - \max_i (x-y)_i \geq -\max_i x_i + \min_i y_i$,   
\begin{align*}
\varphi(x,0) - \varphi(x,-1)   &=    \alpha \int e^{- \frac{\|y\|^2}{2\sigma^2}}  \max_i x_i - \max_i(x-y)_i   dy  \geq \alpha\int e^{- \frac{\|y\|^2}{2\sigma^2}}  \min_i y_i dy  = -\sigma \mathbb E \max_{ i}  G_i   
\end{align*}
where  $\sigma = \sqrt {2\kappa}$ at $t = -1$. Thus, $\varphi(x+r,0) - \varphi(x,-1) \geq -2 -   \sqrt  {2\kappa} \mathbb E \max_ i  G_i$. Similarly, since $ - \max_i (x-y)_i \leq -\max_i x_i + \max y_i$, we obtain $\varphi(x+r,0) - \varphi(x,-1) \leq  2 +   \sqrt  {2\kappa } \mathbb E \max_ i  G_i$. 

\subsubsection{Bound on $|\varphi_{tt}|$ }
\label{app:phi_time_deriv_bound}

For each $t<0$,  it suffices to give a uniform upper bound of $|u_{tt}(x,t)|$ over all $x\in \mathbb R^N$. Since
\[
\partial_{tt}\varphi= \partial_t (- \kappa \Delta u) = -\kappa \Delta  (\partial_t \varphi)  = \kappa^2\Delta^2\varphi
\]
it suffices to bound $\Delta^2\varphi=\sum_{i,j}\partial_{iijj} \varphi$. By Appendix \ref{app:heat_derivs}
\begin{align*}
\sum_{i,j}\partial_{iijj}\varphi&=\sum_i\partial_{iiii}\varphi+\sum_{j\neq i}\partial_{iijj}\varphi\\   
&=\sum_i\frac{\alpha}{\sigma^4}\int_{\mathbb{R}^N}e^{-\frac{\|y\|^2}{2
\sigma^2}}y_i\left (N+2-\frac{\|y\|^2}{\sigma^2}\right)\mathbb{1}_{x_i-y_i>\max_{k\neq i}x_k-y_k}dy\\
&=\frac{\alpha}{\sigma^4}\int_{\mathbb{R}^N}e^{-\frac{\|y\|^2}{2
\sigma^2}}\left (N+2-\frac{\|y\|^2}{\sigma^2}\right)\sum_iy_i\mathbb{1}_{x_i-y_i>\max_{k\neq i}x_k-y_k}dy
\end{align*}
Combining above with the fact that $|\sum_iy_i\mathbb{1}_{x_i-y_i>\max_{k\neq i}x_k-y_k}|\leq\max_i|y_i|$
\begin{align*}
\Big |\sum_{i,j}\partial_{iijj}\varphi \Big |\leq&\frac{1}{\sigma^3}\mathbb{E}\Big [|N+2-\|G\|^2|\max_i|G_i| \Big ]
\end{align*}
Therefore, $|\partial_{tt}\varphi (x,t)|\leq\frac{1}{|t|^{\frac{3}{2}}} \frac{\sqrt{\kappa }}{2 \sqrt 2} \mathbb{E}\Big [|N+2-\|G\|^2|\max_i|G_i|$\Big].

\subsubsection{Bound on $|D^3 \varphi [q,q,q]|$}
\label{app:phi_third_deriv_lb_err}
For each $t<0$,  we bound  $| D^3\varphi( x, t) [q,q,q]|$ uniformly in $x\in \mathbb R^N$ and $q\in[-1,1]^N$. First, note that
\begin{align*}
D^3 \varphi [q,q,q]=&\sum_i(\partial_{iii} \varphi q^2_i+3\sum_{j\neq i}\partial_{ijj} \varphi q^2_j)q_i+\sum_i\sum_{j\neq i}\sum_{k\neq  i,j}\partial_{ijk} \varphi q_iq_jq_k\\
&=\sum_i(-2\partial_{iii} \varphi q^2_i+3\sum_{j}\partial_{ijj}\varphi q^2_j)q_i+\sum_i\sum_{j\neq i}\sum_{k\neq  i,j}\partial_{ijk} \varphi q_iq_jq_k
\end{align*}
We derive the following identity by linearity of $\varphi$ along $\mathbb 1$: 
\begin{align*}
\sum_i\sum_{j\neq i}\sum_{k\neq  i,j}\partial_{ijk}\varphi=&-\sum_i\sum_{j\neq i}(\partial_{ijj}\varphi +\partial_{iij}\varphi)= -2\sum_i\sum_{j\neq i}\partial_{iij}\varphi =2\sum_i\partial_{iii}\varphi
\end{align*}

Using the fact that $\partial_{ijk}\varphi>0$ and this identity, for $N \geq 3$, 
\begin{align*}
\left |D^3 \varphi [q,q,q] \right |&\leq  2\sum_i \left |\partial_{iii}\varphi \right |+3  \sum_i \Big |\sum_{j}\partial_{ijj}\varphi q^2_j \Big|+\sum_i\sum_{j\neq i}\sum_{k\neq  i,j}\partial_{ijk}\varphi\\
&=2 \sum_i \left |\partial_{iii}\varphi \right |+3  \sum_i \Big |\sum_j\partial_{ijj}\varphi q^2_j \Big|+2\sum_i \partial_{iii}\varphi \\
&\leq 3  \sum_i \Big |\sum_j\partial_{ijj}\varphi q^2_j \Big|+4\sum_i |\partial_{iii}\varphi|
\end{align*}
and for $N =2$, 
\begin{align*}
|D^3 \varphi [q,q,q]|\leq 2\sum_i|\partial_{iii} \varphi|+3  \sum_i \Big |\sum_j\partial_{ijj}\varphi q^2_j \Big|
\end{align*}

Using the formulas for third derivatives,
\begin{align*}
\sum_j\partial_{ijj}\varphi q^2_j =-\frac{c_N}{\sigma^2}\int e^{-\frac{\|y\|^2}{2
\sigma^2}}\left (\sum_j q_j^2 \left(1-\frac{y_j^2}{\sigma^2}\right ) \right)\mathbb{1}_{x_i-y_i>\max_{k\neq i}x_k-y_k}dy
\end{align*}
we obtain
\begin{align*}
 \sum_i \big |\sum_j\partial_{ijj}\varphi q^2_j \big| \leq&\frac{c_N}{\sigma^2}\int e^{-\frac{\|y\|^2}{2
\sigma^2}}\Big |\sum_j q_j^2 \Big(1-\frac{y_j^2}{\sigma^2}\Big ) \Big| dy\\
=&\frac{1}{\sigma^2}\mathbb{E}\Big |\sum_j q_j^2 \left(1-G_j^2\right ) \Big|
\end{align*}

Using Jensen's inequality and the independence of $G_j$,
\begin{align*}
\mathbb{E}\left |\sum_j q_j^2 \left(1-G_j^2\right ) \right| \leq \sqrt {\mathbb{E}\left (\sum_j q_j^2 \left(1-G_j^2\right ) \right)^2 } \\ 
= \sqrt {Var \left (\sum_j q_j^2 G_j^2\right ) } =  \sqrt { 2 \sum_j q_j^4  } \leq \sqrt{2N}
\end{align*}

Also, 
\begin{align*}
\sum_i|\partial_{iii} \varphi |\leq&\frac{\alpha}{\sigma^2}\int_{\mathbb{R}^N}e^{-\frac{\|y\|^2}{2
\sigma^2}}\sum_i \Big |1-\frac{y_i^2}{\sigma^2} \Big|\mathbb{1}_{x_i-y_i>\max_{j\neq i}x_j-y_j}dy\\
\leq&\frac{\alpha}{\sigma^2}\int_{\mathbb{R}^N}e^{-\frac{\|y\|^2}{2
\sigma^2}}\max_i|1-\frac{y_i^2}{\sigma^2}|dy\\
=&\frac{1}{\sigma^2}\mathbb{E}[\max_i|1-G_i^2|]
\end{align*}

Therefore, for all $q \in [- 1,1]^N$,  $| D^3 \varphi(x, t) [q,q,q]|  \leq   \frac{1} {|t| } C_3$ 
where $C_3 = \frac {1}{\kappa} \left (\frac{3}{\sqrt 2} \sqrt{N}+  a\mathbb{E} \max_i|1-G_i^2| \right)$ and $a = 1$ for $N =2$ and $a=2$ for $N \geq 3$.

\subsubsection{ Bound of $|D^{4} \varphi [q,q,q,q]|$ for $q \in \{\pm 1\}$. }
\label {app:phi_fourth_deriv_lb_err}

For each $t<-1$, we bound $D^4 \varphi [q,q,q,q]$ uniformly for all $x\in \mathbb R^N$ and $q\in\{\pm 1\}^N$.  For distinct $i,j$ and $k$ by Appendix \ref{app:heat_derivs} we have
\begin{align*}
\partial_{ijii}\varphi + \partial_{ijjj}\varphi + \sum_{k \neq i,j} \partial_{ijkk}\varphi = \frac{\alpha}{\sigma^4}\int_{\mathbb{R}^N}e^{-\frac{\|y\|^2}{2\sigma^2}}y_i \left (N+2-\frac{ \|y\|^2}{\sigma^2}\right )\mathbb{1}_{x_j-y_j>\max_{m\neq j}x_k-y_k}dy    
\end{align*}

Also, 
\begin{align*}
\sum_{ i} \sum_{ j}| \partial_{ijii}\varphi|  
\leq  \sum_{ i} \frac{\alpha}{\sigma^4}\int_{\mathbb{R}^N}e^{-\frac{\|y\|^2}{2
\sigma^2}} \left |y_i \left (3-\frac{y_i^2}{\sigma^2} \right) \right | dy
 \leq \frac{N}{\sigma^3}  \mathbb E  \Big[ | G (3-  G^2)|  \Big  ]
\end{align*}

Since $\partial_{ijkl} \varphi<0$ for distinct $i,j,k,l$  (assuming $N \geq 4$) and $D^4 \varphi [\mathbb 1,\mathbb 1,\mathbb 1,\mathbb 1]=0$.
\begin{align*}
D^4 \varphi [q,q,q,q] \geq&   \sum_i \partial_{iiii} \varphi+3\sum_{i } \sum_{j \neq i}\partial_{iijj} \varphi + 2 \sum_{i }\sum_{j\neq i}(\partial_{ijii} \varphi+ \partial_{ijjj} \varphi) q_i q_j \\
&+ 6  \sum_{i }\sum_{j\neq i} \sum_{k\neq i,j} \partial_{ijkk} \varphi q_i q_j + \sum_{i }\sum_{j\neq i} \sum_{k\neq i,j} \sum_{l\neq i,j,k} \partial_{ijkl} \varphi \\
=&  2 \sum_{i }\sum_{j\neq i}(\partial_{ijii} \varphi+ \partial_{ijjj} \varphi) (q_i q_j-1) + 6  \sum_{i }\sum_{j\neq i} \sum_{k\neq i,j} \partial_{ijkk} \varphi (q_i q_j-1)\\
=&-4 \sum_{i } \sum_{j \neq i }(\partial_{iiij} \varphi +\partial_{ijjj} \varphi)(q_i q_j-1) +  6  \sum_{i }\sum_{j\neq i} \partial_{ij}\Delta \varphi(q_i q_j-1)  \\
\geq&-16\sum_{i } \sum_{j }|\partial_{iiij} \varphi|-24\sum_{i }\sum_{j} |\partial_{ij}\Delta \varphi|\\
\geq& -\frac{16N}{\sigma^3}  \mathbb E  \Big [\Big | G (3-  G^2)  \Big | \Big ] -24  \frac{\alpha}{\sigma^4}\int_{\mathbb{R}^N}e^{-\frac{\|y\|^2}{2
\sigma^2}} \sum_{i }\left | y_i \left (N+2-\frac{ \|y\|^2}{\sigma^2}\right )\right |dy \\
=& -\frac{8}{\sigma^3} \Big (
  2N \mathbb E \Big [ \left | G (3-  G^2)  \right | \Big] + 3 \mathbb E \Big [ \sum_{i } \left | G_i \left (N+2- \|G\|^2\right  ) \right | \Big] \Big )\\
\geq&-\frac{8N}{\sigma^3}(2\sqrt {\mathbb E \Big [G^2 \Big ] \mathbb E \Big [(3 - G^2)^2 \Big] } + 3 \sqrt {\mathbb E [G^2]\mathbb E [\left (N+2- \|G\|^2 \right  )^2] })\\
\geq&-\frac{2\sqrt{2}N}{(\kappa |t+1|)^\frac{3}{2}}(2\sqrt{6}+3\sqrt{2N+4})
\end{align*}
For $N=2,3$ the calculation is similar.

\subsection{ Heat Potential: Upper Bound Error Term}
\label{app:ub_err}
To apply Theorem \ref{thm:fixed_ub} with respect to the player associated with the heat potential $\varphi$, we also need to determine the error term  $E^{\varphi}_{u.b.}(t) = C+ \sum _{\tau = t}^{-2}  K(\tau)$ where $C$ is a constant satisfying   $\max_{a} \mathbb E_{a, p_{-1}} ~\varphi  ( x + r,  0)  - \varphi(x,-1)  \leq C$ for all $x $ and K is a function K 
 \[ 
 -  \frac {1}{2} \text{ess inf}_{\bar \tau \in [\tau, \tau+1] }  w_{tt}(x,\bar \tau) -\frac {1}{6}   \text{ess inf}_{y \in [x, x -q]} ~  D^3w(y,\tau+1) [ q,q,q]  \leq K(\tau)
\] for all $\tau \in [t,-2]$, all $q \in [- 1,1]^N$ and all $x$.

In Appendix \ref{app:lb_err}, we showed that $| \varphi(x,-1) - \varphi(x+r,0)|  \leq C$ for all $x$ and $r$ where  $C = O( \sqrt  {\kappa  \log N})$.   Similarly, in that section we proved that  $|\varphi_{tt}(x,  t) | \leq \frac{K_2}{ |t|^{\frac{3}{2}}}$ for all $x$ and $t < - 1$ where $K_2 = O(\sqrt {\kappa N \log N})$. Finally, we showed that $| D^3 \varphi [q,q,q](x, t) |  \leq   \frac{1} {|t|} K_3$ for all $q \in[-1,1]^N$ where $K_3 = O\left (\frac{\sqrt {N}}{ \kappa}\right)$.
These results are also applicable in the upper bound setting and therefore for $\kappa = 1$, $E^{\varphi}_{u.b.}(t) = O(\sqrt {N \log N}+ \sqrt{N} \log |t|)$.

\section{Comb Adversary}
\label{app:comb_adversary}
In this Appendix, we show that $\Delta \varphi  \leq \langle D^2 \varphi \cdot q^c, q^c \rangle$.  Appendix \ref{app:heat_ordering} shows that if $x_i \geq x_j \geq x_l$, then $ \partial_{ij} \varphi \leq \partial_{il} \varphi \leq 0 $. Using this result, Appendix \ref{app:heat_off-diagonal} shows that $\sum_{i<j} \partial_{ij} \varphi q^c_i q^c_j \geq 0$, which implies the desired result.

\subsection{Ordering of Mixed Derivatives of the Heat Potential}
 \label{app:heat_ordering}
 Note that $\int_{x_j-x_i+y_i}^{\infty}e^{-\frac{y_j^2}{2 \sigma^2}}y_jdy_j = \sigma^2 e^{-\frac{(x_j-x_i +y_j)^2}{2 \sigma^2}}$, and\
 \begin{align*}
&\int^{x_i-\max_{k\neq i,j} (x_k-y_k)}_{-\infty}e^{-\frac{y_i^2}{2 \sigma^2}} \left ( \sigma^2 e^{-\frac{(x_j-x_i +y_j)^2}{2 \sigma^2}} \right) dy_i \\
&= \sigma^3 \frac {\sqrt {\pi}}{2} e^{-\frac {(x_i - x_j)^2}{4 \sigma^2}}   \left  [\text{erf} \left[ \frac {x_i -x_j +2y_i}{2 \sigma} \right ] \right ]_{y_i=- \infty}^{x_i - \max_{k\neq i,j }(x_k -y_k)}   \\
&= \sigma^3 \frac {\sqrt {\pi}}{2} e^{-\frac {(x_i - x_j)^2}{4 \sigma^2}}  \text{erf} \left[ \frac {x_i +x_j -2 \max_{k\neq i,j } (x-y)_k}{2 \sigma} +1 \right ]  
 \end{align*}

Plugging the above into the expression for $\partial_{ij} \varphi$ in Appendix \ref{app:heat_derivs} for $i \neq j$, we obtain
\begin{align*}
\partial_{ij} \varphi(x,t) &=  -\frac{c_N}{\sigma^2}  \int_{\mathbb R^{N-2}} e^ {- \frac {\sum_{k \neq i,j } y_k^2}{2 \sigma^2}} \int_{-\infty}^{x_i - \max_{k \neq i,j } (x_k-y_k)} e^{- \frac{y_i^2}{2 \sigma^2} }  \left (\int_{x_j -x_i+y_i}^{\infty} e^{- \frac {y_j^2}{2 \sigma^2}} y_j  dy_j   \right)dy_i  d\hat y_{i,j}\\
&=   -  {c_N} \sigma \frac {\sqrt {\pi}}{2} e^{-\frac {(x_i - x_j)^2}{4 \sigma^2}}  \int_{\mathbb R^{N-2}} e^{- \frac {\sum_{k=1 }^{N-2} z_k^2}{2 \sigma^2}} \text{erf} \left[ \frac {x_i +x_j -2 \max_{1 \leq k \leq N-2 } (\hat x_k-z_k)}{2 \sigma} +1 \right ] dz
\end{align*} 
 where $\hat x$ is a vector in $\mathbb R^{N-2}$ containing the same components as $x$ except for $x_i$ and $x_j$. 
 
 Let $\varphi =\varphi(x,t) $ be evaluated at  arbitrary $x$ and $t<0$ and let $\{(i)\}_{\in [N]}$ be  the ranked coordinates defined in Section \ref{sec:heat_fixed} associated with $x$. 
Showing that if $x_i \geq x_j \geq x_l$, then $ \partial_{ij} \varphi \leq \partial_{il} \varphi \leq 0 $ is equivalent to showing that if $i \geq j \geq l$, then $\partial_{(i)(j)} \varphi(x,t) \leq \partial_{(i)(l)} \varphi(x,t) \leq 0$. 
 
 Note that if $i \geq j \geq l$,   then $x_{(i)} + x_{(j)} -2 \max_{ k \neq i,j } ( x_{(k)}-z_{(k)}) \geq x_i+ x_l -2 \max_{ k \neq i,l } ( x_{(k)}-z_{(k)})$ for all $z \in \mathbb R^{N-2}$. Since $\text{erf}$ is an increasing function,   $\partial_{(i)(j)} \varphi(x,t) \leq \partial_{(i)(l)} \varphi(x,t) \leq 0$, as desired.

\subsection{Sign of $\sum_{i<j} \partial_{ij} \varphi q^c_i q_j$} 
\label{app:heat_off-diagonal} 
We show that for $q^c$ chosen in accordance with the comb strategy $a^c$, $\sum_{i<j} \partial_{ij} \varphi(x,t) q^c_i q^c_j = \sum_{i<j} \partial_{(i)(j)} \varphi(x,t) q^c_{(i)} q^c_{(j)} \geq 0$ (where the left hand side uses coordinates in an arbitrarily indexed canonical basis and the right-hand side uses ranked coordinates associated with $x$).  If $N$ is even, 
 \begin{align*}
 \sum_{i<j} \partial_{(i)(j)} \varphi q^c_{(i)} q^c_{(j)} =& \sum_{i: \text{odd}} \left (\left (\sum_{i < j: \text{even} <N} - \partial_{(i)(j)} \varphi +   \partial_{(i)(j+1)} \varphi \right) - \partial_{(i)(N)}\varphi \right)\\
 & +\sum_{i: \text{even}} \left  (\sum_{i < j: \text{odd} <N} - \partial_{(i)(j)} \varphi +   \partial_{ij+1} \varphi \right)  
 \end {align*}
 Similarly, if $N$ is odd, 
 \begin{align*}
 \sum_{i<j} \partial_{(i)(j)} \varphi q^c_{(i)} q_{(j)} =& \sum_{i: \text{odd}}\left (\sum_{i < j: \text{even} <N} - \partial_{(i)(j)} \varphi +   \partial_{(i)(j+1)} \varphi \right)\\
 &  +\sum_{i: \text{even}}  \left ( \left  (\sum_{i < j: \text{odd} <N} - \partial_{(i)(j)} \varphi +   \partial_{(i)(j+1)} \varphi \right) - \partial_{(i)(N)}\varphi \right) 
 \end {align*}
 Both of these expressions are  positive by the ordering of mixed partial derivatives established in  Appendix \ref{app:heat_ordering} and the fact that $\partial_{(i)(N)} \varphi < 0$ for $i \neq N$ as shown in Appendix \ref{app:heat_derivs}. 

 \section{Lower Bound Heat Potential: Diffusion Factor}
\label{app:heat_lb_kappa}
Note that $\varphi(x+ c \mathbb 1, t)=\varphi(x, t)+c$ implies that  $ \sum_{i} \partial_i \varphi=1$, $\partial_{ii}\varphi = - \sum \partial_{ij}\varphi$,  and therefore $D^2 \varphi \cdot \mathbb 1 = 0$. For $N=2$, this result and the fact that  $D^2u$ is symmetric imply that $D^2u$  has the form
\[
   D^2u=
  \left[ {\begin{array}{cc}
   a & -a \\
   -a & a \\
  \end{array} } \right]
\] It is straightforward to verify that $\frac{1}{2}\mathbb E_{a^h}  \langle D^2 \varphi \cdot q, q \rangle = \Delta\varphi$ and therefore  $\kappa_h=1$.

When $N>2$, $\frac{1}{2}\mathbb E_{a^h}  \langle D^2 \varphi \cdot q, q \rangle =  \frac{1}{|S|}\sum_{q \in S}\langle D^2 \varphi \cdot q, q \rangle=\langle D^2 \varphi,\frac{1}{|S|}\sum_{q\in S} q q ^\top \rangle_F$ where the set $S$ is defined in Section \ref{sec:heat_fixed} and $\langle\ ,\rangle_F$ is the Frobenius inner product. Since $S$ is permutation invariant, the off-diagonal entries of  $\frac{1}{|S|}\sum_{q\in S}qq^\top$ are all equal and the diagonal entries are all equal to $1$, and therefore, this expression is equal to $(1-\lambda) I+\lambda M$ for some constant $\lambda$ where $M=\mathbb 1\mathbb 1^\top$.    Note that 
\begin{align*}
\frac{1}{|S|}\sum_{q\in S}\langle q q^\top, M\rangle_F=   
\begin{cases}
1&\text{if}~ N~ \text{is odd} \\
0&\text{if}~ N~ \text{is even}\\
\end{cases}
\end{align*}
which implies that $\lambda=   -\frac{1}{N}$ if $N$ is odd and $-\frac{1}{N-1}$ if $N$ is even. Using the fact that $\langle D^2 \varphi, M\rangle_F=0$, we obtain $\langle D^2 \varphi,\frac{1}{|S|}\sum_{q\in S}qq^\top \rangle_F= (1-\lambda)\Delta u$. This shows that $
\frac{1}{2}\mathbb E_{a^h}  \langle D^2 \varphi \cdot q, q \rangle = \kappa_h \Delta\varphi$ where $\kappa_h = \frac{1}{2}(1-\lambda)$, as desired.

The foregoing proof is short and elementary. But to put the result in context, the only properties  $D^2 \varphi $ we used is that it is symmetric and has $\mathbb 1$ in the kernel. Therefore, for an arbitrary $N \times N$ matrix $M$ with these properties, we showed that
\begin{align}
2 \kappa_h \text{Trace} (M) = \mathbb E_{a^h}  \langle M \cdot q, q \rangle \leq  \max_{q \in \{\pm 1\}^N} \langle  M \cdot q, q \rangle
\label{eq:heat_adversary_expectation}
\end{align}
where the inequality follows from a probabilistic argument. 

The Laplacian $L(G)$ of an undirected graph $G = (V, E)$ with $|V| = N$ vertices is given by
\[
L(G)_{ij} = 
\begin{cases}
 -w_{ij} &\text{if}~  i \neq j \\
\sum_{k \in {N}} w_{ik} &\text{if}~  i= j 
 \end{cases}
 \]
 where $w_{ij} \geq 0$ is the weight of the edge $(i,j) \in E$. The sum of the edge weights of $G$ is  $|E| = \sum_{i <j} w_{ij} u =  \frac {1} {2} \text{Trace} (L(G))$ and the \textit {maximum cut} of $G$ is  $\text{maxcut}(G) = \max_{q \in \{\pm 1\}^N} \sum_{i<j}  w_{ij} \frac {1 -q_iq_j}{2} $. Using the convention that $w_{ii} = 0$,
\begin{align*}
&\max_q  \langle L(G) \cdot q, q \rangle  = \max_q \sum_{i,j} - w_{ij}q_iq_j + \sum_{i} \Big [ \sum_j  w_{ij} \Big]  q_{i}^2 \\
& = \max_q \sum_{i,j}  w_{ij} \Big (1 -q_iq_j \Big) = 2 \max_q \sum_{i<j}  w_{ij} \left (1 -q_iq_j \right) = 4 \text{maxcut}(G)
\end{align*}
where  the feasible set of $q$ is $\{\pm 1\}^N$. 

For a graph $G_u$ with each $w_{ij} \in \{0,1\}$ (unweighted graph), it is known that $ \left ( \frac{1}{2} + \frac {1}{2N} \right) |E| \leq \text{maxcut}(G_u)$ \citep{haglin}. Since every Laplacian is symmetric and has $\mathbb 1$ in the kernel, the inequality \eqref{eq:heat_adversary_expectation} implies $ \kappa_h |E| \leq \text{Max-Cut}(G)$ for a weighted graph. (Although, similarly to a graph Laplacian, the off-diagonal elements of $D^2 \varphi $ are negative as shown in Appendix \ref{app:heat_derivs}, we did not use this property in our proof.\footnote{Therefore, our result is broader and also holds for matrices with arbitrary signs of off-diagonal elements, such as Laplacians of graphs with signed edge weights.})

\section{Upper Bound Heat Potential: Diffusion Factor }
\label{app:heat_ub}

Appendix \ref{app:heat_derivs} shows that $\partial_{ij} \varphi< 0$ for $i \neq j$ and $\partial_{ii} \varphi> 0$. Also the fact $ \sum_{i}  \partial_i \varphi =1$, which follows from linearity of $\varphi$ in the direction of $\mathbb 1$, implies that $\sum_{i,j}\partial_{ij} \varphi =0$. Thus, $ \frac {1} {2}  \max_{q \in [-1,1]^N}  \langle D^2 \varphi \cdot q, q \rangle \leq    \frac {1} {2}  \Delta \varphi  - \frac {1} {2}  \sum_{i\neq j}  \partial_{ij} \varphi =   \Delta \varphi$.

\section{Proof of Claim \ref{cl:max_pde_soln}}
\label{app:max_pde}
We prove Claim \ref{cl:max_pde_soln} as follows. In Appendix \ref{app:Derivatives of the max-based potential}, we  compute the spatial derivatives of max potential $\psi$ defined by \eqref{eq:max_pde_soln} up to the third order for every $x$ in the ranked coordinates $\{(i)\}_{i\in [N]}$, as defined in Section \ref{sec:heat_fixed}. In Appendix \ref{app:maxC2} we prove that when the ranking changes, the second derivatives are continuous, and therefore, $\psi$ is a $C^2$ function of $x$.   The third order spatial derivatives are defined almost everywhere (i.e., everywhere except where the ranking changes) and bounded. Therefore,  the  second order derivatives of $\psi$ are Lipschitz continuous but $\psi$ is not a $C^3$ function of $x$.  Finally, in Appendix \ref{app:max_pde_soln}, we use these results to show that $\psi$ satisfies \eqref{eq:max_pde}.

\subsection{Derivatives of the Max Potential}
\label{app:Derivatives of the max-based potential}
Note that 
\begin{align*}
f'(z) = \text{erf} \left ( \frac {z}{\sqrt {2}} \right) ~~&\text {and}~~ f''(z) =  \sqrt {\frac{2} {\pi  }} \exp \left (- \frac {z^2}{2} \right)
\end {align*}
Then for $i\leq j$
\begin{align*}
\partial_{(i)} f (z_l) =
 \begin {cases}
 0  &\text{~if~}   l + 1< i   \\
  -\frac { l}{\sqrt {-2 \kappa t}} f'(z_l) & \text{~if~}  l +1 = i  \\
 \frac {1}{\sqrt {-2 \kappa t}} f'(z_l) & \text{~if~}   l \geq i  
 \end{cases} 
~~&\text{and}~~~ \partial_{(i)(j)} f (z_l) =
 \begin {cases}
 \frac { 1}{(-2 \kappa t)} f''(z_l) &\text{~if~}  j \leq  l \\
  \frac { l^2 }{(-2 \kappa t)} f''(z_l) & \text{~if~}  i =j =  l+1  \\
 \frac {l}{2 \kappa t} f''(z_l) &\text{~if~}  i < j = l+1 \\
 0  &\text{~if~}   j  >  l+1  
 \end{cases}
 \end{align*} 
Therefore, the first derivatives are
\begin{align*}
\partial_{(i)} \psi &= 
 \begin {cases}
\frac {1}{N} + \sum_{l=1}^{N-1} c_l f'(z_l)  &\text{~if~} i=1  \\
\frac {1}{N} + \sum_{l=i}^{N-1} c_l f'(z_l) - (i-1) c_{i-1}f'(z_{i-1})   &\text{~if~} i \geq 2 
\end{cases}
\end{align*}
Since $x_{(1)} \geq x_{(2)} \geq ... \geq x_{(N)}$,  we have $0 \leq z_1 \leq z_2 \leq ... \leq z_{N-1}$ and therefore $ 0\leq f'(z_1) \leq f'( z_2) \leq ... \leq f'(z_{N-1})$. As a consequence $\partial_i \psi \geq 0, \forall i\in[N]$.

The second derivatives are
\begin{align*}
\partial_{(i)(i)} \psi &= 
 \begin {cases}
 \frac {1 }{\sqrt {-2 \kappa t}}  \sum_{l=1}^{N-1} c_l f''(z_l)  &\text{~if~} i =1  \\
\frac {1 }{\sqrt {-2 \kappa t}} \left (\sum_{l= i}^{N-1} c_l f''(z_l) + (i-1)^2c_{i-1}f''(z_{i-1}) \right)   &\text{~if~} 2 \leq  i \leq N-1 \\
\frac {1 }{\sqrt {-2 \kappa t}}  (N-1)^2c_{N-1}f''(z_{N-1})    &\text{~if~} i= N 
\end{cases}
\end{align*} 
or for $i<j$
\begin{align*}
\partial_{(i)(j)} \psi &= 
 \begin {cases}
 \frac {1 }{\sqrt {-2 \kappa t}} \left ( \sum_{l= j}^{N-1} c_l f''(z_l)   -  (j-1) c_{j-1} f''(z_{j-1}) \right) &\text{~if~}  j <N   \\
- \frac {1 }{\sqrt {-2 \kappa t}} ( N-1) c_{N-1} f''(z_{N-1}) &\text{~if~}  j =N   
\end{cases}
\end{align*} 

The third derivatives are
\begin{align*}
\partial_{(i)(i)(i)}\psi = 
\begin{cases}
\frac {1}{(-2 \kappa t)} \sum_{l =1}^{N-1} c_l f'''(z_l) & \text {if}~ i=1 \\
\frac {1}{(-2 \kappa t)} \left( \sum_{l =i}^{N-1} c_l f'''(z_l) - (i-1)^3 c_{i-1} f'''(z_{i-1}) \right)& \text {if}~ 2 \leq i \leq N-1\\
\frac {1}{2 \kappa t}  (N-1)^3 c_{N-1} f'''(z_{N-1}) & \text {if}~ i = N \\
\end{cases}
\end{align*}
 when $i\neq j$,
\begin{align*}
\partial_{(i)(j)(j)}\psi = 
\begin{cases}
\frac {1}{(-2 \kappa t)} \left( \sum_{l =j}^{N-1} c_l f'''(z_l) + (j-1)^2 c_{j-1} f'''(z_{j-1}) \right)& \text {if}~ i<j \leq N-1\\
\frac {1}{(-2 \kappa t)} (N-1)^2 c_{N-1} f'''(z_{N-1})& \text {if}~  i<j=N\\
\frac {1}{(-2 \kappa t)} \left( \sum_{l =i}^{N-1} c_l f'''(z_l) - (i-1) c_{i-1} f'''(z_{i-1}) \right)& \text {if}~ j<i \leq N-1\\
\frac {1}{(2 \kappa t)} (N-1) c_{N-1} f'''(z_{N-1})& \text {if}~  j<i=N
\end{cases}
\end{align*}
and when $i<j<k$
\begin{align*}
\partial_{(i)(j)(k)}\psi = 
\begin{cases}
\frac {1}{(-2 \kappa t)} \left( \sum_{l =k}^{N-1} c_l f'''(z_l) - (k-1) c_{k-1} f'''(z_{k-1}) \right)& \text {if}~ k \leq N-1\\
\frac {1}{2 \kappa t} (N-1) c_{N-1} f'''(z_{N-1})& \text {if}~  k= N
\end{cases}
\end{align*}

\subsection{$\psi$ is $C^2$ with Lipschitz Continuous Second Order Spatial Derivatives} 
\label{app:maxC2}

First, we show that the function $\psi$ defined by \eqref{eq:max_pde_soln} is $C^2$ in the spatial
variables $x_1, \ldots, x_N$. Since \eqref{eq:max_pde_soln} uses ranked coordinates, we can view
$\psi$ as being defined in the sector $\{ x_1 \geq x_2 \geq \cdots \geq x_N \}$ then extended
by symmetry to all ${\mathbb R}^N$.

The heart of the matter is the observation that {\it at each plane $x_k = x_{k+1}$ the normal
derivative of $\psi$ is zero}. Indeed, when $x_1 \geq x_2 \geq \cdots \geq x_N $ the formula \eqref{eq:max_pde_soln} involves two sums, $\sum_{i=1}^N x_i$ and $\sum_{l=1}^{N-1} c_l f(z_l)$.
The former certainly has normal derivative equal to zero at each of the sector's faces $x_k=x_{k+1}$,
so we may concentrate on the latter. Since $z_1 = x_1-x_2$ while $z_2, \ldots, z_{N-1}$ are symmetric
in $x_1$ and $x_2$, at the face $x_1=x_2$ (equivalently, $z_1 = 0$) the normal derivative is
a multiple of $f'(0)$, which vanishes since $f(z)$ is an even function of $z$
(see \eqref{eq:sturm_ode_soln}). Turning to the face $x_k = x_{k+1}$ with $k \geq 2$, we observe that
$z_1, \ldots, z_{k-2}$ do not involve $x_k$ or $x_{k+1}$ while $z_{k+1}, \ldots, z_{N-1}$ are
symmetric in $x_k$ and $x_{k+1}$; moreover $x_k = x_{k+1}$ is equivalent to $z_{k-1} = z_k$. Therefore
the normal derivative of $\psi$ is a multiple of
$$
c_k f'(z_k) - (k-1) c_{k-1} f'(z_k) + kc_k f'(z_k) = 0
$$
using the fact that $c_k = \frac{1}{k(k+1)}$.

To explain why this observation implies the $C^2$ continuity of $\psi$, it suffices to consider the
restriction of $\psi$ to $\{ x_1 + \cdots + x_N = 0 \}$ (since $\psi(x_1 + c, \dots, x_N + c, t)=\psi(x, t)+c$).
Changing variables to $y_k = x_k - x_{k+1}$ ($1 \leq k \leq N-1$), the $C^2$ character of
$\psi$ follows from the following calculus lemma applied to
$g(y_1, \ldots y_{N-1}) = \psi(x_1, \ldots, x_N, t)$ for any fixed $t$.

\begin{lemma}
For any $m \geq 1$, let $g(y_1, \ldots, y_m)$ be $C^2$ on the positive quadrant
$\{ y_i \geq 0 \ \mbox{for each $i$} \}$, and assume that $\partial_i g = 0$ at the face
$y_i =0$. Then the symmetric extension of $g$, 
$$
g (y_1, \ldots, y_m) = g(|y_1|, \ldots, |y_m|) ,
$$
is $C^2$ on all ${\mathbb R}^m$.
\end{lemma}
\begin{proof}
The case $m=1$ is familiar: for $y_1 < 0$ we have $g'(y_1) = - g'(-y_1)$ and
$g''(y_1) = g''(-y_1)$. If $g'(0)=0$ then $g$ and its first and second derivatives match
at $y_1 = 0$, and it follows that $g$ is $C^2$.

The case $m=2$ similar. At the face $y_1=0$ of the positive quadrant we
have $\partial_1 g (0,y_2) = 0$ by hypothesis, and therefore $\partial_{12} g(0,y_2) = 0$
by differentiation with respect to $y_2$; similarly, $\partial_{12} g = 0$ at the face
$y_2 = 0$. It follows that the first and second derivatives of the extension of $g$ are all continuous
across the planes $y_1 = 0$ and $y_2 = 0$. So $g$ is $C^2$.

The general case is essentially the same. To see that $\partial_i g$ and $\partial_{ii} g$
are continuous it suffices to apply the argument used for $m=1$ along the line obtained by holding
all variables except $y_i$ constant. To see that $\partial_{ij} g$ is continuous for $i \neq j$
it suffices to apply the argument used for $m=2$ in the plane obtained by holding all variables
except $y_i$ and $y_j$ constant.
\end{proof}

We next show $\psi$ is not $C^3$. Suppose $x_1>x_2>x_3>x_4...>x_N$ then since
$\psi(x_1,x_2,x_3...,x_N)=\psi(x_1,x_3,x_2...,x_N)$ we have $\partial_{222}\psi(x_1,x_2,x_3...,x_N)=\partial_{333}\psi(x_1,x_3,x_2...,x_N)$. However 
\[
\partial_{222}\psi(x_1,x_2,x_3...,x_N)-\partial_{333}\psi(x_1,x_2,x_3...,x_N)=\frac{3}{2}f'''(z_2)-\frac{1}{2}f'''(z_1)
\]
which does not approach to 0 when $x$ approaches to $\{x_1>x_2=x_3>...>x_N\}$. This means $\partial_{333}\psi$ cannot be continuously extended to the boundary $\{x_1>x_2=x_3>...>x_N\}$. 

Finally, we show the boundedness of third order derivatives. Note that for $z \geq 0$,
\[
-  \sqrt {\frac {2}{ e \pi }} \leq f'''(z) =  - \sqrt {\frac{2} {\pi  }} ze^ {- \frac {z^2}{2}} \leq 0
\]
From Appendix \ref{app:Derivatives of the max-based potential} we have
\begin{align*}
\begin{cases}
\frac {1}{2 \kappa t}  \sqrt {\frac {2} {e \pi }}\left (\frac {1}{i} -\frac {1}{N}  \right ) \leq \partial_{(i)(i)(i)} \psi  \leq \frac {1}{(-2 \kappa t)}\sqrt { \frac { 2}{ e \pi }}\frac {(i-1)^2}{i}\\
\frac {1}{2 \kappa t}\sqrt { \frac { 2}{ e \pi }}(1-\frac{1}{N})\leq \partial_{(i)(j)(j)} \psi \leq 0& \text {if}~  i<j\\
\frac {1}{2 \kappa t}  \sqrt {\frac {2} {e \pi }}\left (\frac {1}{i} -\frac {1}{N}  \right )\leq \partial_{(i)(j)(j)} \psi \leq \frac {1}{(-2 \kappa t)}\sqrt { \frac { 2}{ e \pi }}\frac {1}{i}& \text {if}~  i>j\\
\frac {1}{2 \kappa t}  \sqrt {\frac {2} {e \pi }}\left (\frac {1}{k} -\frac {1}{N}  \right )\leq \partial_{(i)(j)(k)} \psi \leq \frac {1}{(-2 \kappa t)}\sqrt { \frac { 2}{ e \pi }}\frac {1}{k}& \text {if}~  i<j<k
\end{cases}
\end{align*}

\subsection{Max Potential $\psi$ Satisfies \eqref{eq:max_pde} }
\label{app:max_pde_soln}
First, note that $\lim_{z \rightarrow \infty}  f \left (z \right )/z = 1$, and therefore, the final value condition is satisfied.  
\begin{align*}
\lim_{t \rightarrow 0} \psi(x,t) &=  \frac {1}{N} \langle x, \mathbb 1 \rangle +  \sum_{i =1} ^{N-1} \frac {1}{i(i+1)}  \left ( \left ( \sum_{j=1}^i x_{(j)}\right)  - i x_{(i)+1} \right)  \\
 & = \frac {1}{N} \langle x, \mathbb 1 \rangle +  \sum_{j=1}^{N-1}x_{(j)} \left ( \sum_{i =j}^{N-1} \frac {1}{i(i+1)}  \right )-\sum_{i =1} ^{N-1}   \frac { x_{(i)+1} }{i+1} \\
  & = \frac {1}{N} \langle x, \mathbb 1 \rangle +  \sum_{j=1}^{N-1}x_{(j)} \left ( \frac {1}{j} - \frac {1}{N} \right )-\sum_{i =2} ^{N}   \frac { x_{(i)} }{i} \\
  & = \frac {1}{N} \langle x, \mathbb 1 \rangle + x_{(1)} -  \left (\sum_{j=1}^{N-1} \frac {x_{(j)}}{N} \right) -   \frac { x_{(N)} }{N}   \\
  & = x_{(1)}=\max_i(x_i)
\end{align*}

Since $x_{(1)} \geq x_{(2)} \geq ... \geq x_{(N)}$,  we have $0 \leq z_1 \leq z_2 \leq ... \leq z_{N-1}$ and, therefore, $ \sqrt{\frac{2}{\pi}} \geq  f''(z_1) \geq f''( z_2) \geq ... \geq f''(z_{N-1})\geq 0$. This by a straightforward computation gives for $i\leq N-1$, 
\[
\partial_{(i)(i)}\psi-\partial_{(i+1)(i+1)}\psi=(1-\frac{1}{i})(f''(z_{i-1})-f''(z_i))\geq 0
\]
Therefore, $\max_{ i }  \partial^2_i \psi   = \partial^2_{(1)} \psi =\partial^2_{(2)} \psi$.  Finally,  $\psi_t =  -\frac{ \sqrt {\kappa }}{\sqrt {-2t}} \sum_{l=1}^{N-1} c_l  f''(z_l)$ and thus $\psi_t + \kappa  \partial^2_{(1)} \psi=  0$.

\section{Lower Bound Max Potential: Diffusion Factor } 
\label{app:max_lb} 
The linearity in the direction of $\mathbb 1$ implies that  $ \sum_{i} \partial_i \psi =1$, and therefore $D^2 \psi \cdot \mathbb 1 = 0$.  Suppose  $x_{(1)}=x_i$, in Appendix \ref{app:max_pde_soln} we show that  $\max_{ j }  \partial^2_j \psi   = \partial^2_i \psi$. Therefore, $\pm\langle D^2 \psi \cdot q^m,  q^m \rangle  = \pm \langle D^2 \psi \cdot \pm (q^m,  q^m) \rangle = 4 \partial_{ii} \psi  = 4 \max_j \partial_{jj} \psi$.

\section{Upper Bound Max Potential: Diffusion Factor $\kappa_m$}
\label{app:max_kappa} 
We note that since $f$ is convex, $\psi$ is convex. Therefore, $\max_{q \in [-1 ,1]^N} \langle D^2 \psi \cdot q, q \rangle$ is attained at the vertices of the hypercube $ \{\pm 1\}^N$.  
Without loss of generality we assume $x_1\geq x_2...\geq x_N$. From Appendix \ref{app:max_pde}, we see that $D^2 \psi$ has a special structure: $\partial_{ik} \psi = \partial_{jk}\psi $ for all $i, j<k$ and $\partial_{ij} \psi \leq \partial_{ik} \psi \leq 0$ for $i<j<k$.  In the remainder of this Appendix we use this structure to prove that a class of simple rank-based strategies maximizes the quadratic form $\max_{q\in\{\pm 1\}^N}\langle D^2\psi \cdot q, q\rangle$,\footnote{This class includes the comb strategy.} and compute the $\kappa_m$ such that 
\[
\frac{1}{2}\max_{q\in\{\pm 1\}^N}\langle D^2 \psi\cdot q, q\rangle \leq\kappa_m\max_i\partial_{ii}\psi
\]
 From Appendix \ref{app:Derivatives of the max-based potential} we know that for $i<j$ $\partial_{ij}\psi$ is a function of $j$ alone, thus we denote $a_j=-\partial_{ij}\psi$ for any $i<j$. Also,
\begin{align*}
\partial_{ij}\psi= \frac {1 }{\sqrt {-2 \kappa t}}  \left ( \sum_{l= j}^{N-1} c_l f''(z_l)  -  (j-1) c_{j-1} f''(z_{j-1}) \right)\leq\frac {1 }{\sqrt {-2 \kappa t}}\frac{ f''(z_j)- f''(z_{j-1})}{j}\leq 0
\end{align*}
and for $i<j<k$ 
\begin{align*}
\partial_{ij}\psi-\partial_{ik}\psi=& \frac {1 }{\sqrt {-2 \kappa t}} \left ( \left ( \sum_{l= j}^{k-1} c_l f''(z_l) \right)  -  (j-1) c_{j-1} f''(z_{j-1})+(k-1) c_{k-1} f''(z_{k-1}) \right)\\
\leq&\frac {1 }{\sqrt {-2 \kappa t}}\left (  f''(z_j)\left (\frac{1}{j}-\frac{1}{k}\right) -  \frac{f''(z_{j-1})}{j} +\frac{f''(z_j)}{k}  \right)\leq 0
\end{align*}
thus $a_2\geq a_3...\geq a_N\geq 0$.\\

\begin{theorem}\label{thm:optimal strategy for max potential}
For the max potential $\psi$ on $\{x|x_1\geq x_2...\geq x_N\}$, $\max_{q\in\{\pm 1\}^N}\langle D^2\psi\cdot q, q\rangle$ is obtained by strategies satisfying $q_{2i-1}+q_{2i}=0$, $\forall 2i\leq N$. Specifically, comb strategy $q^c$ achieves the maximum.  
\end{theorem} 

\begin{proof}
As shown in Appendix \ref{app:heat_lb_kappa}, we can view $D^2\psi$ as the Laplacian of an undirected weighted graph $G$ with $N$ vertices. The edge weight $w_{ij}=-\partial_{(i)(j)}\psi=a_j$ for $i<j$ and $a_2\geq a_3...\geq a_N$. Also, as shown
\begin{align*}
\max_{q\in\{\pm 1\}^N}\langle D^2\psi\cdot q, q\rangle=4max\_cut(G)
\end{align*}
Thus, we converted the problem of maximizing a quadratic to the problem finding the max cut for a special weighted graph. The Theorem proved below gives us the desired result. 
\end{proof}

\begin{theorem}\label{thm:max-cut theorem}
Consider an undirected graph with vertices $\{1,...,N\}$ satisfying for any edge $(i,j)$ the weight depends on $\max(i,j)$, i.e. we can write $w_{ij}=a_j$ for $i<j$. Also suppose $a_2\geq a_3...\geq a_N$ , then the max cut, modulo permutations between vertices $(i,j)$ such that $a_i=a_j$, is any cut dividing $2i-1$ and $2i$ for all $1\leq i\leq \lfloor\frac{N}{2}\rfloor$.
\end{theorem}
\begin{proof}
Without loss of generality, assume $a_2>a_3...>a_N$. We use induction on N. For $N=2$ and $N=3$ it is straight forward to check that the max cut is any cut dividing 1 and 2.\\

For $N+1$ points, we first prove the max cut must divide 1 and 2. 
\begin{lemma}\label{max cut lemma}
Any max cut must divide 1 and 2.
\end{lemma}

\begin{proof}[Proof of lemma \ref{max cut lemma}]
Assume a max cut doesn't divide 1 and 2, denote 
\begin{align*}
\begin{cases}
L=\{i\in\{3,...,N\}|\text{i on the same side as 1 and 2}\}\\
R=\{i\in\{3,...,N\}|\text{i on the other side}\}
\end{cases}
\end{align*}
by definition $R$ is nonempty.\\

Define $A_L=\sum_{j\in L}a_j$ and $A_R=\sum_{j\in R}a_j$. If $A_R<A_L+a_2$ then by moving 2 to $R$ the cut will get bigger since 
\begin{align*}
T(\{1\}\cup L,\{2\}\cup R)=T(\{1,2\}\cup L, R)+a_2+A_L-A_R> T(\{1,2\}\cup L, R) 
\end{align*}
which is a contradiction.\\

So $A_R\geq A_L+a_2$.  We denote $p_i=\{2i-1,2i\}$, $2\leq i\leq \lfloor\frac{N}{2}\rfloor$. If no $p_i$ satisfies $p_i\subset R$, then 
\begin{align*}
A_R-A_L\leq (a_3-a_4)+(a_5-a_6)+...<a_2 
\end{align*}
which is a contradiction. Thus, we can assume $p_k$ is the smallest set contained in $R$. We prove that by moving 2 to $R$ and $2k-1$ to $L$ the cut will get bigger. Actually
\begin{align*}
T(\{1,2k-1\}\cup L,\{2\}\cup R\setminus\{2k-1\})=&T(\{1,2\}\cup L,R)+a_2+(|R_{2k-1}|-|L_{2k-1}|-1)a_{2k-1}\\
&+A_{L_{2k-1}}-A_{R_{2k-1}}
\end{align*}
where
\begin{align*}
\begin{cases}
L_{2k-1}=L\cap\{3,...,2k-1\}\\
R_{2k-1}=R\cap\{3,...,2k-1\}
\end{cases}
\end{align*}
and $A_{L_{2k-1}}$, $A_{R_{2k-1}}$ are defined under the same convention as $A_L$, $A_R$.\\

By definition of $k$ for any $p_i$ such that $2\leq i\leq k-1$, if one of the element is in $R_{2k-1}$ then the other must be in $L_{2k-1}$. Suppose $R_{2k-1}$ contains elements of $p_{i_1},...,p_{i_{|R_{2k-1}|-1}}$ and $2k-1$, then  
\begin{align}
&a_2+(|R_{2k-1}|-|L_{2k-1}|-1)a_{2k-1}+A_{L_{2k-1}}-A_{R_{2k-1}} \nonumber\\
~~&\geq a_2+\sum_{j=1}^{|R_{2k-1}|-1}a_{2i_j}-\sum_{j=1}^{|R_{2k-1}|-1}a_{2i_j-1}-a_{2k-1} \label{first_line}\\
~~&~~~~+\left (\sum_{l\in L_{2k-1}\setminus\cup_{j=1}^{|R_{2k-1}|-1} p_{i_j}}a_l \right) \label{second_line}\\
~~&~~~~-(|L_{2k-1}|-|R_{2k-1}|+1)a_{2k-1} \label{third_line}
\end{align}
We can rearrange the sum in \eqref{first_line}
\begin{align*}
&a_2+\sum_{j=1}^{|R_{2k-1}|-1}a_{2i_j}-\sum_{j=1}^{|R_{2k-1}|-1}a_{2i_j-1}-a_{2k-1}\\
=&(a_2-a_{2i_1-1})+(a_{2i_1}-a_{2i_2-1})+...(a_{2i_{|R_{2k-1}|-1}}-a_{2k-1})>0 
\end{align*}
Also notice that each $a_l>a_{2k-1}$ for $l\in L_{2k-1}\setminus\cup_{j=1}^{|R_{2k-1}|-1} p_{i_j}$ and 
\[
|L_{2k-1}\setminus\cup_{j=1}^{|R_{2k-1}|-1} p_{i_j}|=|L_{2k-1}|-|R_{2k-1}|+1
\]
which implies that \eqref{second_line} plus \eqref{third_line} is positive. This demonstrates that the new cut is strictly bigger which is a contradiction. 
\end{proof}

Returning to the proof of Theorem \ref{thm:max-cut theorem}, denote
\[
S_i=\{j\in\{3,...,N\}|\text{$j$ on the same side as $i$}\}  
\]
for $i=1,2$. 

Also denote $T(A,B)$ as the total weights of edges between $A$ and $B$. Then
\begin{align*}
 T(\{1\}\cup S_1,\{2\}\cup S_2)=\sum_{i=3}^Na_i+T(S_1,S_2)
\end{align*}
Thus $(S_1,S_2)$ must be the max cut for $\{3,...,N\}$ as well. By induction hypothesis the max cut divides $2i-1$ and $2i$ for $2\leq i\leq [\frac{N}{2}]$.
\end{proof}

Now we use Theorem \ref{thm:optimal strategy for max potential} to compute $\kappa_m$. Using the same notation $a_i$ as above, since the comb strategy $q^c$ attains the maximum, 
\begin{align*}
\max\langle D^2 \psi\cdot q, q\rangle=&\langle D^2 \psi \cdot q^c, q^c\rangle
=&
\begin{cases}
\sum_{i=1}^{M-1}4i(a_{2i}+a_{2i+1})+4Ma_{2k}& N=2M\\
\sum_{i=1}^M4i(a_{2i}+a_{2i+1})& N=2M+1
\end{cases}
\end{align*}
Notice that
\begin{align*}
\max_i\partial_{ii}\psi=\partial_{11}\psi=\sum_{i=2}^Na_i
\end{align*}
Taking
\begin{align*}
\kappa_m=&\max_{a_2\geq a_3...\geq a_N\geq 0}\frac{\frac{1}{2}\langle D^2\psi\cdot q^c, q^c\rangle}{\partial_{11}\psi}=
\begin{cases}
\frac{N^2}{2(N-1)}&N\ even\\
\frac{N+1}{2}&N\ odd
\end{cases}
\end{align*}
the max is obtained when $a_2=a_3...=a_N$.

\section{Max Potential Error Terms}
\label{app:max_bound} 

In this Appendix, we compute the error terms for the heat potential $\psi$ given by \eqref{eq:max_pde_soln}.   In Appendix \ref{app:lb_err_max}, we determine the lower bound error $E^{\psi}_{l.b.}$ for $\psi$ with $\kappa =2$ associated with the adversary $a^m$, and in Appendix \ref{app:ub_err_max}, we determine the upper bound error $E^{\psi}_{u.b.}$ with $\kappa=\kappa_m$ given by \eqref{eq:lower_bound_max_factor}.

\subsection{Max Potential: Lower Bound Error}
\label{app:lb_err_max}

To apply Theorem \ref{thm:fixed_lb} with respect to  the max potential $\psi$, and the associated adversary $a^m$ we determine the ``error" term  $E^{\varphi}_{l.b.}(t) = C_{l.b.}+ \sum _{\tau = t}^{-2}  K_{l.b.}(\tau)$ where  $C_{l.b.}$ is a constant satisfying  $\psi(x,-1) - \min_{p} \mathbb E_{a_{-1},p} ~\psi (x + r,  0)  \leq C_{l.b.}$ for all $x$, and $K_{l.b.}$ is a function satisfying
 \begin{align*}
 \frac {1}{2}\text{ess sup}_{  \bar \tau \in [\tau, \tau+1] }   \psi_{tt}(x,\bar \tau ) + \frac {1}{6}   \text{ess sup}_{y \in [x, x \mp q^m]} ~\pm D^3\psi(y,\tau+1) [q^m,q^m,q^m]   \leq K_{l.b.}(\tau)
  \end{align*}
for all $\tau \in [t,  -2]$ and all $x$.

In Appendix \ref{app:max_final_period}, we show that $ \psi(x,-1) - \psi(x+r,0)  \leq C_{l.b.}$ for all $x$ and $r$ where $C_{l.b.} = 2 +2\sqrt {\frac{ \kappa}{\pi}}  \frac {N-1}{N} $.  
In Appendix \ref {app:psi_time_deriv_bound}, we prove that  $\psi_{tt}(x,  \tau)  \leq \frac{K^{l.b.}_2}{ |\tau|^{\frac{3}{2}}}$ for all $x$ and $\tau \leq - 1$ where $K^{l.b.}_2 =   \frac{N-1} {N}   \frac {\sqrt {\kappa}}{2 \sqrt {\pi  }}$. Finally, in Appendix \ref {app:psi_third_deriv_lb_err}, we show that $\text{ess sup}_{y  \in [x, x\mp q^m]}  \pm D^3\psi( y, t+1) [ q^m, q^m, q^m] \leq   \frac{1} {|t|} K^{l.b.}_3$  where $K^{l.b.}_3 = \frac {4}{ \kappa } \frac {(N-1)^2}{N}\sqrt { \frac { 2}{ e \pi }}$.  Therefore,  $K_{l.b.}(\tau)  = \frac {1}{2} \frac{K^{l.b.}_2}{ |\tau+1|^{\frac{3}{2}}}+  \frac {1}{6} \frac{1} {|t+1|} K^{l.b.}_3$ and 
\begin{align*}
\sum_{\tau = t}^{-2} K_{l.b.}(\tau) \leq   \frac{K^{l.b}_2}{2} \left (3 - \frac{2}{\sqrt{|t| - 1}}\right) + \frac{K^{l.b}_3} {6} (1+\log(|t| - 1))
 \end{align*}
The foregoing shows that for $\kappa=2$, $E^{\psi}_{l.b.}(t) = O(N  \log |t|)$.

\subsubsection{Bounds on $\psi(x,-1) - \psi(x+r,0)$}
\label{app:max_final_period} 

We decompose the difference as follows
\[
\psi(x+r,0) - \psi(x,-1)  = \max_i(x+r)_i -  \max_ix_i +\psi(x,0) - \psi(x,-1) 
\]
Since $r=q_I\mathbb 1-q \in[-2,2]^N$, we obtain $-2\leq \max_i(x+r)_i -  \max_i(x_{-1})_i\leq 2$. Also, for any $x$,
\[
\psi(x,0) - \psi(x,-1)  = x_{(1)} -\frac {1}{N} \sum_{l=1}^N x_{(l)} -    \sqrt {2 \kappa} \sum_{l=1}^{N-1} c_l f(z_l)
=  \sqrt {2 \kappa} \sum_{l=1}^{N-1} c_l z_l -    \sqrt {2 \kappa} \sum_{l=1}^{N-1} c_l f(z_l)
\]
Since $-\sqrt{\frac{2}{\pi}} \leq z -f(z) \leq 0$ for $z \geq 0$,
\begin{align*}
 -2\sqrt {\frac{\kappa}{\pi}}  \frac {N-1}{N}  \leq \psi(x,0) - \psi(x,-1) \leq 0 
\end{align*}
 This implies that
\begin{align*}
-2 -2\sqrt {\frac{ \kappa}{\pi}}  \frac {N-1}{N}  \leq \psi(x+r,0) - \psi(x,-1)\leq 2  
\end{align*}

\subsubsection{Bounds on $\psi_{tt}(x,\tau)$}
\label{app:psi_time_deriv_bound}
We have
\begin {align*}
\psi_{tt}& = \frac{  \sqrt{\kappa}}{2 \sqrt {2} (-\tau)^{\frac{3}{2}}} \sum_{l=1}^{N-1} c_l  f''(z_l)   +\frac{ \sqrt{\kappa}}{\sqrt {-2\tau}} \sum_{l=1}^{N-1} c_l  f'''(z_l) \frac {z_l}{-2\tau} \\
&=\frac{ \sqrt{\kappa}}{2 \sqrt {2} (-\tau)^{\frac{3}{2}}} \sum_{l=1}^{N-1} c_l  \left (f''(z_l)   + f'''(z_l) z_l \right)\\
&=\frac{ \sqrt{\kappa}}{2 \sqrt {2} (-\tau)^{\frac{3}{2}}} \sum_{l=1}^{N-1} c_l  \left (1   - z_l^2 \right)   \sqrt {\frac{2} {\pi  }} e^ {- \frac {z_l^2}{2}} 
\end{align*}
Note that for all $z$, $-  2 e ^{-\frac 3 2}  \leq \left (1-{z^2}  \right) e^{- \frac {z^2}{2} } \leq 1$. Therefore, for all $x$ and $\tau \leq - 1$,
\[
-\frac{ 1}{(-\tau)^{\frac{3}{2}}} \frac{N-1} {N} \sqrt{   \frac{\kappa} {e^3 \pi  }}\leq \psi_{tt}  \leq\frac{ 1}{(-\tau)^{\frac{3}{2}}} \frac{N-1} {N}   \frac {\sqrt {\kappa}}{2 \sqrt {\pi  }}
\]

\subsubsection{Upper Bound of $ \text{ess sup}_{y  \in [x, x\mp q^m]} \pm D^3\psi( y, t+1) [ q^m, q^m, q^m] $}
\label {app:psi_third_deriv_lb_err}
Without loss of generality assume $x_1\geq x_2\geq...\geq x_N$, then $q^m=(1,-1...,-1)$ and $q$ in the support of $a^m$ is either $q^m$ or $-q^m$. We give an upper bound of 
\[
\text{ess sup}_{y  \in [x, x\mp q^m]} \pm D^3\psi( y, t+1) [ q^m, q^m, q^m]
\]

Since $\psi$ is linear along $\mathbb 1$, $
D^3\psi( y, t+1) [ q^m, q^m, q^m]=D^3\psi( y, t+1) [ q^m+\mathbb 1, q^m+\mathbb 1, q^m+\mathbb 1]=8\partial_{111}\psi(y,t+1)$. If $q=-q^m$, then $[x,x+q^m]\subset\{x|x_1\geq x_2\geq...\geq x_N\}$. For $y\in[x,x+q^m]$ 
\begin{align*}
&D^3\psi(y,t+1)[-q^m,-q^m,-q^m]=-8\partial_{111}\psi(y,t+1)\\
&=-8\partial_{(1)(1)(1)}\psi(y,t+1)\leq \frac {4}{- \kappa (t+1)} \sqrt {\frac {2} {e \pi }} \frac {N-1}{N} 
\end{align*}

If $q=q^m$, suppose $x_2+1\geq x_3+1...x_k+1\geq x_1-1\geq x_{k+1}+1...\geq x_N+1$, $k$ ranges from 1 to $N$.  We can accordingly partition $[x,x-q^m]$ into $k$ subintervals $I_1...I_k$ such that $y_1$ ranks $l$'s for $y\in I_l$. Thus, in each subinterval  
\[
D^3\psi(y,t+1)[q^m,q^m,q^m]=8\partial_{111}\psi(y,t+1)=8\partial_{(l)(l)(l)}\psi(y,t+1)\leq\frac {4}{- \kappa (t+1)} \frac {(l-1)^2}{l}\sqrt { \frac { 2}{ e \pi }} 
\]

Summarizing the above, we have
\[
\text{ess sup}_{y  \in [x, x\mp q^m]}  D^3\psi( y, t+1) [\pm q^m,\pm q^m,\pm q^m]\leq \frac {4}{- \kappa (t+1)} \frac {(N-1)^2}{N}\sqrt { \frac { 2}{ e \pi }}
\]

\subsection{Max Potential: Upper Bound Error}
\label{app:ub_err_max}
To apply Theorem \ref{thm:fixed_ub} with respect to the max potential $\psi$, we also need to determine the error term  $E^{\psi}_{u.b.}(t) = C_{u.b.}+ \sum _{\tau = t}^{-2}  K_{u.b.}(\tau)$ where $C_{u.b.}$ is a constant satisfying   $\max_{a} \mathbb E_{a, p_{-1}} ~\psi  ( x + r,  0)  - \psi(x,-1)  \leq C_{u.b.}$ for all $x $ and $K_{u.b.}$ is a function $K_{u.b.}$. 
 \[ 
 -  \frac {1}{2} \text{ess inf}_{\bar \tau \in [\tau, \tau+1] }  w_{tt}(x,\bar \tau) -\frac {1}{6}   \text{ess inf}_{y \in [x, x -q]} ~  D^3w(y,\tau+1) [ q,q,q]  \leq K(\tau)
\] for all $\tau \in [t,-2]$, all $q \in [- 1,1]^N$ and all $x$.

Appendix \ref{app:max_final_period} showed that $  \varphi(x+r,0)- \varphi(x,-1) \leq C_{u.b.}$ for all $x$ and $r$ where  $C_{u.b.} = 2$.   Also, Appendix \ref{app:psi_time_deriv_bound} proved that  $-\psi_{tt}(x,  \tau)  \leq \frac{K^{u.b.}_2}{ |\tau|^{\frac{3}{2}}}$ for all $x$ and $\tau \leq - 1$ where $K^{u.b.}_2 = \frac{N-1} {N} \sqrt{   \frac{\kappa} {e^3 \pi  }}$. Finally, below we show that $| D^3 \psi [q,q,q](x, t) |  \leq   \frac{1} {|t|} K^{u.b.}_3$ for all $q \in[-1,1]^N$ where $K^{u.b.}_3 = O\left (\frac{N^2}{ \kappa}\right)$. Therefore, for $\kappa =\kappa_m $, $E^{\varphi}_{u.b.}(t) = O(N \log |t|)$. 

In the remaining part of this Appendix, we show that 
\[
| D^3\psi [q,q,q]| \leq \frac {1}{2\kappa |\tau|}\sqrt {\frac { 2}{ e \pi }}\left (\frac{7}{2}N^2-8N+5\log N+\frac{3}{2}\right)
\]
uniformly over all $q\in[-1,1]^N$  and $x\in \mathbb R^N$. By Appendix \ref{app:Derivatives of the max-based potential}, 
\begin{align*}
\begin{cases}
|\partial_{(1)(1)(1)} \psi| \leq \frac {1}{-2 \kappa (t+1)}\sqrt {\frac { 2}{ e \pi }}(1-\frac{1}{N})\\
|\partial_{(i)(i)(i)} \psi| \leq \frac {1}{-2 \kappa (t+1)}\sqrt {\frac { 2}{ e \pi }}\frac {(i-1)^2}{i}& \text {if}~  i>1\\
|\partial_{(i)(j)(j)} \psi| \leq \frac {1}{-2 \kappa (t+1)}\sqrt { \frac { 2}{ e \pi }}(1-\frac{1}{N})& \text {if}~  i<j\\
|\partial_{(i)(j)(j)} \psi |\leq \frac {1}{-2 \kappa (t+1)}\sqrt { \frac { 2}{ e \pi }}\frac {1}{i}& \text {if}~  i>j\\
|\partial_{(i)(j)(k)} \psi| \leq \frac {1}{-2 \kappa (t+1)}\sqrt { \frac { 2}{ e \pi }}\frac {1}{k}& \text {if}~  i<j<k
\end{cases}
\end{align*}
Notice that for any $i,j,k$, $\partial_{(i)(j)(k)}\psi$ only depends on $\max(i,j,k)$, for $q\in[-1,1]^N$ we have 
\begin{align*}
|D^3\psi( x, t+1)& [q,q,q]|\leq\sum_{i=1}^N|\partial_{(i)(i)(i)}\psi(x,t+1)|+\sum_{i=2}^N(\sum_{j=1}^{i-1}q_j^2)|\partial_{(i)(j)(j)}\psi(x,t+1)|\\
+&\sum_{j=2}^N|\sum_{i=1}^{j-1}q_i||\partial_{(i)(j)(j)}\psi(x,t+1)|+6\sum_{k=3}^N(\sum_{i=1}^{k-1}q_i)(\sum_{j=1}^{k-1}q_j)|\partial_{(i)(j)(k)}\psi(x,t+1)|\\
\leq&\frac {1}{-2 \kappa (t+1)}\sqrt {\frac { 2}{ e \pi }}\left(N-1+\sum_{i=2}^N(1-\frac{1}{i})+\sum_{j=2}^N(j-1)(1-\frac{1}{N})+6\sum_{k=3}^N\frac{(k-1)^2}{k}\right)\\
\leq& \frac {1}{-2 \kappa (t+1)}\sqrt {\frac { 2}{ e \pi }}\left(\frac{7}{2}N^2-8N+5\log N+\frac{3}{2}\right)
\end{align*}

\section{Numerical Computation of Bounds} 
\label{app:numerics}
In this Appendix, we describe numerical computation of bounds obtained by $a^s$, $a^h$ and $p^h$ that are presented in Figures \ref{fig:l} and \ref{fig:2}.

The lower bound attained by $a^s$, as rescaled for our losses, is
\[
\sum_{j=0}^{M-1}\mathbb E \left [ \Big| \sum_{\substack{1 \leq t \leq |T| \\{t \mod M=j}}} Z_t  \Big| \right] 
\]
where $M = \floor{\log_2 N}$ and each $Z_t$ is an independent Radamacher random variable. As noted in the same reference $ \mathbb E \left [ | \sum_{t \in[n]} Z_t|\right] \leq  \sqrt {\frac {2n}{\pi}} \exp \left ( \frac{1}{12n} - \frac{2}{6n+1} \right ) $, and we will set the expected distance of each random walk to be equal to its upper bound for comparison purposes. 

The bounds obtained by using $a^h$ and $p^h$ are expressed in terms of $\mathbb E_G \max G_i$ where $G$ is a standard N-dimensional Gaussian.  Note that $\mathbb E_G \max G_i = \int_{-\infty}^\infty t \frac {d}{dt} (\Phi (t)^N) dt$ where $\Phi$ is the c.d.f. of the Gaussian random variable $N(0,1)$. Therefore, for comparison purposes, we evaluate the expectation of the maximum of Gaussian using numerical integration (\textit{integral} function in MATLAB).

\end{document}